\documentclass[11pt, letterpaper, logo, onecolumn, copyright]{main}

\usepackage[authoryear, sort&compress, round]{natbib}

\usepackage[most, breakable, skins]{tcolorbox}
\tcbuselibrary{skins}
\usepackage{lipsum}
\usepackage{tabularx}
\usepackage{afterpage}
\usepackage{booktabs}
\usepackage{xurl}
\usepackage{subcaption}
\usepackage{makecell}
\usepackage{multirow}
\usepackage{bm}
\usepackage{multicol}
\usepackage{array}
\usepackage{float}
\usepackage{listings}
\usepackage{listings-rust}
\usepackage{fontawesome5}
\usepackage{graphicx}
\usepackage{longtable}
\usepackage{pdflscape}
\usepackage{adjustbox}
\usepackage{CJKutf8}
\usepackage{ragged2e}
\usepackage{colortbl}
\usepackage{enumitem}
\usepackage{bbm}
\usepackage{pifont}
\usepackage{hyphenat}
\usepackage{circledsteps}
\usepackage{diagbox}
\usepackage{textcomp}
\usepackage{mathtools}
\usepackage{threeparttable}
\usepackage{pgfplotstable}
\usepackage{pgfplots}
\pgfplotsset{compat=1.18}

\usepackage{algorithm}
\usepackage{algpseudocode}

\newtheorem{theorem}{Theorem}
\newtheorem{proposition}{Proposition}

\theoremstyle{definition}
\newtheorem{definition}{Definition}

\theoremstyle{remark}

\usepackage{wrapfig}
\usepackage{tikz}
\usepackage[normalem]{ulem}

\definecolor{lightblue}{RGB}{210, 220, 250}
\definecolor{medgray55}{gray}{0.55}
\definecolor{medgray}{gray}{0.7}
\definecolor{litegray}{gray}{0.9}
\definecolor{gblue}{RGB}{210, 227, 252}
\definecolor{gred}{RGB}{250, 210, 207}
\definecolor{gyellow}{RGB}{254, 239, 195}
\definecolor{ggreen}{RGB}{206, 234, 214}
\definecolor{gorange}{RGB}{254, 223, 200}
\definecolor{gblue9}{RGB}{23, 78, 166}
\definecolor{gred9}{RGB}{165, 14, 14}
\definecolor{gyellow9}{RGB}{227, 116, 0}
\definecolor{ggreen9}{RGB}{13, 101, 45}
\definecolor{gorange9}{RGB}{176, 96, 0}
\definecolor{myblue}{rgb}{0,0,1}
\definecolor{myred}{rgb}{1,0,0}
\definecolor{mylightgray}{gray}{0.95}
\definecolor{myCite}{HTML}{1C4587}
\definecolor{highlightblue}{HTML}{185ABC}
\definecolor{cellHighlight}{HTML}{dbefff}
\definecolor{lightgray}{RGB}{211, 211, 211}
\definecolor{lightfont}{gray}{0.3}

\definecolor{lightblue1}{rgb}{0.97, 0.985, 1}
\definecolor{lightblue2}{rgb}{0.92, 0.965, 1}
\definecolor{lightblue3}{rgb}{0.84, 0.93, 1}
\definecolor{lightblue4}{rgb}{0.74, 0.87, 1}
\definecolor{lightblue5}{rgb}{0.64, 0.81, 1}
\definecolor{lightblue6}{rgb}{0.54, 0.75, 1}
\definecolor{lightgreen1}{rgb}{0.97, 1.00, 0.97}
\definecolor{lightgreen2}{rgb}{0.92, 0.98, 0.92}
\definecolor{lightgreen3}{rgb}{0.84, 0.95, 0.84}
\definecolor{lightgreen4}{rgb}{0.74, 0.91, 0.74}
\definecolor{lightgreen5}{rgb}{0.64, 0.86, 0.64}
\definecolor{lightgreen6}{rgb}{0.54, 0.81, 0.54}
\definecolor{lightorange1}{rgb}{1.00, 0.98, 0.95}
\definecolor{lightorange2}{rgb}{1.00, 0.95, 0.85}
\definecolor{lightorange3}{rgb}{1.00, 0.90, 0.70}
\definecolor{lightorange4}{rgb}{1.00, 0.85, 0.55}
\definecolor{lightorange5}{rgb}{1.00, 0.80, 0.40}
\definecolor{lightorange6}{rgb}{1.00, 0.75, 0.30}
\definecolor{lightpurple1}{rgb}{0.985, 0.97, 1.00}
\definecolor{lightpurple2}{rgb}{0.96, 0.92, 1.00}
\definecolor{lightpurple3}{rgb}{0.93, 0.84, 1.00}
\definecolor{lightpurple4}{rgb}{0.87, 0.74, 1.00}
\definecolor{lightpurple5}{rgb}{0.81, 0.64, 1.00}
\definecolor{lightpurple6}{rgb}{0.75, 0.54, 1.00}
\definecolor{lightred1}{rgb}{1.00, 0.97, 0.97}
\definecolor{lightred2}{rgb}{1.00, 0.92, 0.92}
\definecolor{lightred3}{rgb}{1.00, 0.84, 0.84}
\definecolor{lightred4}{rgb}{1.00, 0.74, 0.74}
\definecolor{lightred5}{rgb}{1.00, 0.64, 0.64}
\definecolor{lightred6}{rgb}{1.00, 0.54, 0.54}
\definecolor{lightcyan1}{rgb}{0.97, 1.00, 1.00}
\definecolor{lightcyan2}{rgb}{0.92, 0.98, 0.98}
\definecolor{lightcyan3}{rgb}{0.84, 0.95, 0.96}
\definecolor{lightcyan4}{rgb}{0.74, 0.91, 0.94}
\definecolor{lightcyan5}{rgb}{0.64, 0.87, 0.92}
\definecolor{lightcyan6}{rgb}{0.54, 0.83, 0.90}

\usepackage{minitoc}

\noptcrule

\newcolumntype{L}[1]{>{\raggedright\let\newline\\\arraybackslash\hspace{0pt}}m{#1}}
\newcolumntype{C}[1]{>{\centering}m{#1}}
\newcolumntype{R}[1]{>{\raggedleft\let\newline\\\arraybackslash\hspace{0pt}}m{#1}}

\definecolor{ao}{rgb}{0.0, 0.0, 1.0}

\makeatletter
\renewcommand\subparagraph{%
 \@startsection {subparagraph}{5}{\z@ }{3.25ex \@plus 1ex \@minus .2ex}{-1em}{\normalfont\normalsize\bfseries}}
\makeatother

\let\cite\citep
\hypersetup{citecolor=myCite, linkcolor=myCite, urlcolor=myCite}
\usepackage[capitalise, nameinlink]{cleveref}

\newcommand{\methodname}{\textsc{Papo}}

\newcommand{\myheaderbreak}{\\}

\title{Entropy Polarity in Reinforcement Fine-Tuning:\myheaderbreak Direction, Asymmetry, and Control}

\correspondingauthor={Tao Gui; Fei Huang}

\author{%
  Jiazheng Zhang\textsuperscript{\rm $*$},
  Ziche Fu\textsuperscript{\rm $*$},
  Junrui Shen\textsuperscript{\rm $*$},
  Yunbin Zhao\textsuperscript{\rm $*$},
  Yunke Zhang\textsuperscript{\rm $*$},
  Zhiheng Xi, 
  Long Ma,
  Chenxin An,
  Zhihao Zhang,
  Shichun Liu,
  Dingwei Zhu,
  Shihan Dou,
  Shaofan Liu,
  Han Li, 
  Wiggin Zhou,
  Aiden Adams,
  Tao Gui\textsuperscript{\rm $\dag$},
  Fei Huang\textsuperscript{\rm $\dag$},
  Qi Zhang\textsuperscript{\rm $\dag$},
  Xuanjing Huang\textsuperscript{\rm $\dag$} \\
  \vspace{0.3cm}
  \normalsize
  Fudan NLP Group, Honor Device Co., Ltd \\
  University of Hong Kong, Shanghai Jiao Tong University and Tencent Hunyuan\\
  \texttt{jzzhang24@m.fudan.edu.cn,\{tgui,qz,xjhuang\}@fudan.edu.cn}
}

\vspace{-50pt}
\begin{abstract}
Policy entropy has emerged as a fundamental measure for understanding and controlling exploration in reinforcement learning with verifiable rewards (RLVR) for LLMs.
However, existing entropy-aware methods mainly regulate entropy through global objectives, while the token-level mechanism by which sampled policy updates reshape policy entropy remains underexplored.
In this work, we develop a theoretical framework of entropy mechanics in RLVR.
Our analysis yields a first-order approximation of the entropy change $\Delta \mathcal{H}$, giving rise to entropy polarity, a signed token-level quantity that predicts how much a sampled update expands or contracts entropy.
This analysis further reveals a structural asymmetry: reinforcing frequent high-probability tokens triggers contraction tendencies, whereas expansive tendencies typically require lower-probability samples or stronger distributional correction.
Empirically, we show that entropy polarity reliably predicts entropy changes, and that positive and negative polarity branches play complementary roles in preserving exploration while strengthening exploitation.
Building on these insights, we propose \textbf{P}olarity-\textbf{A}ware \textbf{P}olicy \textbf{O}ptimization (\textbf{\methodname{}}), which preserves both polarity branches and implements entropy control through advantage reweighting.
With the empirical entropy trajectory as an online phase signal, \methodname{} adaptively reallocates optimization pressure between entropy-expanding and entropy-contracting updates.
Experiments on mathematical reasoning and agentic benchmarks show that \methodname{} consistently outperforms competitive baselines, while delivering superior training efficiency and substantial reward improvements.
\end{abstract}

\begin{document}
\begingroup
  \renewcommand\thefootnote{}
  \footnote{\hspace{-1.8em}\textsuperscript{$*$}\,Equal contribution.\quad\textsuperscript{$\dag$}\,Corresponding authors: Tao Gui, Fei Huang, Qi Zhang, Xuanjing Huang.}
\endgroup
\doparttoc
\faketableofcontents

\vspace{-30pt}
\maketitle
\renewcommand{\myheaderbreak}{ }

\section{Introduction}
Reinforcement fine-tuning (RFT) has substantially improved the capabilities of large language models (LLMs)~\citep{Sutton:1998:RL,Guo:2025:R1,Jaech:2024:O1,Xi:2026:RFT}, especially on complex tasks such as mathematics~\citep{Shao:2024:GRPO}, coding~\citep{ClaudeCode}, and agentic decision-making~\citep{Xi:2025:AgentGymRL,Feng:2025:GiGPO,zhang2026agentvrlscalingrewardmodeling}.
A pivotal post-training paradigm driving these gains is Reinforcement Learning with Verifiable Rewards (RLVR)~\citep{Guo:2025:R1,Zhang:2025:ReasoningSurvey, An:2025:Polaris}, where model behavior is optimized against task-specific objectives through automated correctness verification~\citep{zhang:2026:agentv}.
Despite the rapid progress of RLVR, the exploration and exploitation trade-off remains a central challenge for RFT: excessive exploitation can prematurely narrow the policy and limit further discovery, whereas insufficient exploitation can slow or destabilize learning~\citep{Liu:2021:ExplorationExploitation,Tang:2025:SamplePolarity, Huang:2026:DirectionofRLVR}.

As a practical proxy for policy behavior, entropy has become a critical proxy for understanding RLVR~\citep{Cui:2025:EntropyMechanism, Cheng:2026:EntropyAAAI, Chen:2025:RethinkEntropy}.
Recent work has begun to analyze RLVR from this perspective, highlighting the role of high-entropy tokens and exploratory reasoning dynamics~\citep{Wang:2025:8020,Petrenko:2026:EntropyPreserveRL}.
Complementary work has proposed a variety of entropy-aware interventions to mitigate entropy collapse, including entropy regularization~\citep{Prabhudesai:2025:EntropyReg}, clamped bonuses~\citep{Shen:2025:AEnt}, and dynamic clipping~\citep{Xi:2026:BAPO,Chen:2026:ClipControl} for entropy control.
However, an in-depth understanding of entropy mechanics remains unresolved: beyond identifying high-entropy regions or designing heuristic interventions, it is still unclear how policy-gradient updates reshape policy entropy along experience trajectories, how these effects evolve over training, and to what extent they drive RLVR's performance gains.

To this end, we develop a token-level theory of entropy mechanics for RLVR.
Instead of viewing entropy as a global training statistic alone, our framework characterizes the microscopic mechanism linking policy gradient to logit perturbation, next-token probability redistribution, and subsequent policy entropy variation.
At context state \(s_t\), we characterize the policy behavior via the token-level entropy change:
\begin{equation}
\label{eq:delta-H-intro}
\Delta \mathcal{H}_t \;=\; \mathcal{H}\!\left(\pi_{\theta}^{k+1}(\cdot \mid s_t)\right) - \mathcal{H}\!\left(\pi_{\theta}^{k}(\cdot \mid s_t)\right),
\end{equation}
whose sign indicates whether it expands or contracts the policy's exploratory capacity.
We show that \(\Delta \mathcal{H}_t\) admits a first-order decomposition into a sampled-token contribution and a state-wise distributional correction.
This gives rise to a computable signed quantity, \textbf{\emph{entropy polarity}}, that predicts whether and how much reinforcing a given token will increase or decrease local entropy.
The polarity view further reveals an asymmetric property: reinforcing high-probability continuations tends to sharpen the distribution, whereas entropy expansion requires lower-probability samples or stronger distributional correction, explaining why entropy collapse is easier to trigger than to reverse.

Building on the theoretical analysis above, we further connect entropy polarity to practical RLVR training.
Section~\ref{sec:theory_verification} demonstrates that entropy polarity reliably predicts measured entropy changes between consecutive steps, confirming it as a behaviorally meaningful token-level signal.
Section~\ref{sec:empirical_direction} further shows that positive and negative-polarity updates induce distinct dynamics: the former preserve exploratory capacity, while the latter strengthen reward-aligned exploitation.
Since neither branch alone is sufficient for stable RL in Figure~\ref{fig:polarity_directional_control}, an effective RFT strategy should therefore keep both branches while controlling their relative influence over training instead of favoring one entropy direction or treating them uniformly.

To this end, we propose \methodname{}, \textbf{Polarity-Aware Policy Optimization} framework that turns entropy polarity from an analytical quantity into a control signal for RFT. 
Rather than regulating entropy uniformly, \methodname{} directly manages the competition between entropy-expanding and entropy-contracting updates, dynamically redistributing optimization pressure between the two branches according to their effective polarity and the evolving entropy trajectory. In this way, the method preserves exploratory capacity when entropy collapse is imminent, while restoring reward-aligned exploitation when training needs to consolidate useful behaviors. 
Extensive experiments validate the effectiveness of \methodname{} spanning mathematics reasoning and agentic tool-calling. 
Our method consistently outperforms baselines on both in-distribution and out-of-distribution benchmarks.
Overall, our main contributions are:
\begin{itemize}[leftmargin=*, itemsep=2pt, topsep=2pt]
\item We develop a token-level theoretical framework for entropy mechanics in RLVR. This leads to entropy polarity, a computable signed quantity that predicts whether and how much a sampled update tends to expand or contract policy entropy.
\item We demonstrate that entropy polarity reliably predicts entropy change in empirical study, making it a behaviorally meaningful signal. Positive and negative polarity play complementary roles: the former preserve exploration, whereas the latter promote exploitation.
\item We propose \methodname{}, an efficient RL method that turns entropy polarity into an online signal for entropy control. By dynamically reweighting token-level advantages, the proposed method achieves competitive performance on both mathematical reasoning and agentic scenarios.

\end{itemize}

\section{Preliminary}
\label{sec:preliminary}
\paragraph{Group Relative Policy Optimization.} Given a prompt $x$, LLM $\pi_{\theta}$ generates response
$\pi_{\theta}(y\mid x)
=
\prod_{t=1}^{T}
\pi_{\theta}(y_t\mid s_t),$
$s_t := (x,y_{<t})$ denotes the decoding context. 
In GRPO, the rollout policy samples a group of $G$ responses
$\{y^{(i)}\}_{i=1}^{G}$, and each response $y^{(i)}=(y^{(i)}_1,\dots,y^{(i)}_{T_i})$ receives a response-level reward
$r^{(i)}=r(x,y^{(i)})$. It optimizes the clipped surrogate objective
\begin{equation}
\mathcal{J}_{\mathrm{GRPO}}(\theta)
=
\frac{1}{G}
\sum_{i=1}^{G}
\frac{1}{T_i}
\sum_{t=1}^{T_i}
\min
\left(
\rho_{i,t}(\theta) A^{(i)},
\operatorname{clip}\!\left(\rho_{i,t}(\theta),1-\epsilon,1+\epsilon\right) A^{(i)}
\right),
\label{eq:grpo_objective}
\end{equation}
where 
$\rho_{i,t}(\theta)
=
\frac{
\pi_{\theta}(y_t^{(i)}\mid s_t^{(i)})
}{
\pi_{\theta_{\mathrm{old}}}(y_t^{(i)}\mid s_t^{(i)})
}$ is the importance sampling ratio and 
$
A^{(i)}
=
\frac{
r^{(i)}-\mathrm{mean}(\mathbf{R})
}{
\mathrm{std}(\mathbf{R})
}
$ denotes the normalized
advantages within each group.

\paragraph{Policy Entropy.}
In RFT, entropy has become a practical indicator of policy uncertainty.
For a decoding state $s_t$, let $p_v=\pi_\theta(v\mid s_t)$ denote the probability assigned to token $v\in\mathcal{V}$.
The local policy entropy and its change between consecutive steps are defined as
\begin{equation}
\mathcal{H}_t
:=
-\sum_{v\in\mathcal{V}}p_v\log p_v,
\qquad
\Delta\mathcal{H}_t
:=
\mathcal{H}(\pi_{\theta}^{k+1}(\cdot\mid s_t))
-
\mathcal{H}(\pi_{\theta}^{k}(\cdot\mid s_t)).
\label{eq:local_entropy_and_change}
\end{equation}
Here $\pi_\theta^k$ and $\pi_\theta^{k+1}$ denote the policies before and after one optimization step.
A positive $\Delta\mathcal{H}_t$ indicates that the update locally spreads probability mass over alternative continuations, whereas a negative value indicates local concentration around fewer continuations.

\section{Entropy Mechanics}
\label{sec:theory}

We study entropy mechanics through two complementary quantities.
The \emph{intrinsic entropy tendency} \(\mathcal{T}(s_t,y_t)\) captures how reinforcing a sampled token would reshape entropy based on the token and distribution alone, while the \emph{entropy polarity} \(\mathcal{P}(s_t,y_t,A)\) incorporates the signed training signal and captures the realized first-order entropy effect, including both direction and magnitude.

\subsection{A Token-Level View of Entropy Change}
\label{sec:theory_sampled_to_polarity}

We analyze how policy gradient reshapes the next-token entropy. Given a decoding state $s_t$, consider the next-token policy $\pi_\theta(\cdot\mid s_t)$, parameterized by softmax logits $z_{s_t,v}$ over $v\in\mathcal{V}$. Suppose token $y_t\in\mathcal{V}$ is sampled from state $s_t$, and the update induced by policy gradient is from the term $A\log \pi_\theta(y_t\mid s_t)$, where $A\in\mathbb{R}$ is the associated advantage. For compactness, we write
\[
p_v := \pi_\theta(v\mid s_t),
\qquad
p_t := \pi_\theta(y_t\mid s_t),
\qquad
\mathcal{H}_t := \mathcal{H}(\pi_\theta(\cdot\mid s_t)).
\]

\begin{proposition}[Logit Update Under Policy Gradient]
\label{prop:sampled_logit_update}
Under the local softmax policy, one gradient-ascent step on $A\log \pi_\theta(y_t\mid s_t)$ with step size $\eta>0$ gives
$z'_{s_t,v}=z_{s_t,v}+\eta A\bigl(\mathbb{I}[v=y_t]-p_v\bigr)$ for all $v\in\mathcal{V}$.
Equivalently, $z'_{s_t,y_t}=z_{s_t,y_t}+\eta A(1-p_t)$, and
$z'_{s_t,v}=z_{s_t,v}-\eta A p_v$ for $v\neq y_t$.
\end{proposition}

\begin{theorem}[First-order Entropy Change via Sampled Updates]
\label{thm:sampled_entropy_change}
Under the setup of Proposition~\ref{prop:sampled_logit_update}, the entropy change at state $s_t$ satisfies
\begin{equation}
\Delta\mathcal{H}_t
=
-\eta A\,t_1(s_t,y_t)
+
\eta A\,t_2(s_t)
+
O(\eta^2),
\label{eq:entropy_decomposition}
\end{equation}
where
\begin{equation}
t_1(s_t,y_t):=p_t(\mathcal{H}_t+\log p_t),
\qquad
t_2(s_t):=\sum_{v\in\mathcal{V}}p_v^2(\mathcal{H}_t+\log p_v).
\label{eq:t1_t2_def}
\end{equation}
\end{theorem}

\noindent\textit{Proof sketch.}
We derive the result by applying a first-order Taylor expansion to policy entropy, and incorporating the gradient-induced logit update derived in Proposition~\ref{prop:sampled_logit_update}. Under the softmax parametrization,
\begin{equation}
\frac{\partial \mathcal{H}}{\partial z_{s_t,v}}
=
-p_v(\mathcal{H}_t+\log p_v),
\qquad
\delta z_{s_t,v}
=
\eta A\bigl(\mathbb{I}[v=y_t]-p_v\bigr).
\label{eq:entropy_gradient_and_perturbation}
\end{equation}
Substituting $\delta z_{s_t,v}$ into the first-order expansion of $\mathcal{H}$ and collecting the sampled-token term and the distribution term yields Eq.~\eqref{eq:entropy_decomposition}. Full derivations are deferred to Appendix~\ref{app:full_derivations}.

Theorem~\ref{thm:sampled_entropy_change} shows that the first-order entropy change induced by a sampled token update factorizes into the signed training signal \(A\) and an \(A\)-independent token-state quantity \(-t_1(s_t,y_t)+t_2(s_t)\). We therefore take this combination as the mechanism-level object of our analysis.

\begin{definition}[Intrinsic Entropy Tendency]
\label{def:intrinsic_entropy_tendency}
For a sampled token-state pair \((s_t,y_t)\), we define
\begin{equation}
\mathcal{T}(s_t,y_t):=-t_1(s_t,y_t)+t_2(s_t).
\label{eq:structural_polarity}
\end{equation}
\end{definition}

Here, \(\mathcal{T}\) measures the entropy tendency induced by positively reinforcing \(y_t\) at state \(s_t\):
\(\mathcal{T}>0\) indicates entropy expansion, while \(\mathcal{T}<0\) indicates entropy contraction.
Because it is independent of the signed advantage, \(\mathcal{T}\) characterizes the mechanism-level structure before the actual RL update direction is applied.

\begin{figure}[t]
    \centering
    \includegraphics[width=1.0\linewidth]{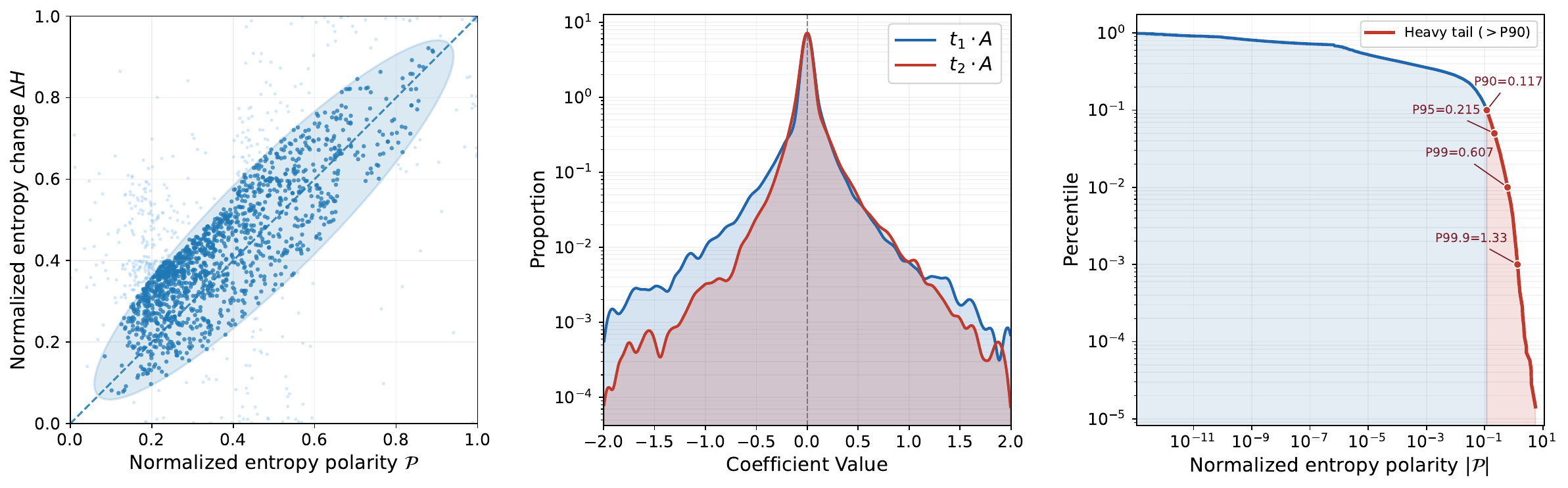}
    \caption{
    (a)~Correlation between polarity score $\mathcal{P}$ and measured entropy change $\Delta \mathcal{H}$.
    (b)~Histograms for the two decomposed components derived from $\mathcal{P}$.
    (c)~Percentile curve of the polarity score.
    }
    \label{fig:polarity_structure}
\end{figure}

\subsection{Intrinsic Entropy Tendency and Asymmetry}
\label{sec:polarity_components}

Recall that the intrinsic entropy tendency is
\(\mathcal{T}(s_t,y_t)=-t_1(s_t,y_t)+t_2(s_t)\), which combines a sampled-token effect and a state-wise distributional correction. 
The distributional term \(t_2(s_t)\) is shared by all sampled tokens at the same state and acts as a correction toward entropy expansion; it is always non-negative, as it can be expressed as the covariance between token probabilities and their log-probabilities.
The sampled-token term enters with a negative sign: since \(t_1(s_t,y_t)=p_t(\mathcal{H}_t+\log p_t)\), its sign is governed by the threshold \(p_t=\exp(-\mathcal{H}_t)\). 
Thus, low-probability tokens make \(-t_1\) expansion-aligned, whereas high-probability tokens make \(-t_1\) contractive.

This sign structure explains the asymmetry of \(\mathcal{T}\). 
For low-probability samples, \(-t_1\) and \(t_2\) jointly favor expansion. 
For high-probability samples, \(-t_1\) competes against the positive correction \(t_2\), and contraction is triggered when the token-specific concentration effect dominates. 
Because policy sampling encounters high-probability tokens more often, reinforcing already dominant tokens tends to further concentrate the local distribution, making contractive tendencies easier to trigger. 
Conversely, expansion typically requires lower-probability samples or a sufficiently large correction term.

This asymmetry is structural: it is already encoded in \(\mathcal{T}\) before the signed training signal is applied; detailed sign analysis and sampled-scale structure are deferred to Appendix~\ref{app:polarity_analysis}.
However, practical RL updates are signed and scaled by the advantage, so structural tendency alone does not yet determine the realized first-order entropy effect of an update.

\subsection{Entropy Polarity and Empirical Insights}
\label{sec:theory_verification}

\paragraph{From tendency to polarity.}
Because actual RL updates are signed by the advantage, we define the entropy polarity of a sampled update as
\begin{equation}
\mathcal{P}(s_t,y_t,A):=A\,\mathcal{T}(s_t,y_t),
\qquad
\Delta\mathcal{H}_t
=
\eta\,\mathcal{P}(s_t,y_t,A)
+
O(\eta^2).
\label{eq:polarity_and_entropy_change}
\end{equation}
Thus, \(\mathcal{P}\) is the realized first-order entropy effect: positive entropy polarity corresponds to an entropy-expanding update, whereas negative entropy polarity corresponds to an entropy-contracting update.

\paragraph{Entropy polarity tracks measured entropy change.}
Theorem~\ref{thm:sampled_entropy_change} predicts that the signed first-order entropy change is governed by the entropy polarity \(\mathcal{P}=A\mathcal{T}\). As shown in Figure~\ref{fig:polarity_structure}(a), normalized \(\mathcal{P}\) aligns closely with the measured local entropy movement during RL training, confirming that entropy polarity is a practical token-level predictor of entropy direction under sampled updates. This agreement suggests that \(\mathcal{P}\) is not merely an analytical artifact, but a reliable signal for analyzing how sampled updates shape local entropy.

\paragraph{Polarity effects are structured and non-uniform.}
Figure~\ref{fig:polarity_structure}(b) plots the two signed components \(t_1\cdot A\) and \(t_2 \cdot A\) that enter the first-order entropy change.
Their substantial overlap indicates that, under practical RLVR training, both the sampled-token contribution and the state-wise distributional correction materially shape the realized entropy effect.
Figure~\ref{fig:polarity_structure}(c) further shows a long-tailed magnitude distribution of \(\mathcal{P}\), where a small fraction of high-magnitude updates accounts for a large share of the cumulative entropy effect.

Together, these results support a two-level view of entropy mechanics:
\(\mathcal{T}\) captures the intrinsic entropy tendency of a sampled token-state pair, while \(\mathcal{P}\) incorporates the signed advantage and captures the realized first-order entropy effect---both direction and magnitude---under actual RL updates.

\begin{figure}[t]
    \centering
    \includegraphics[width=\linewidth]{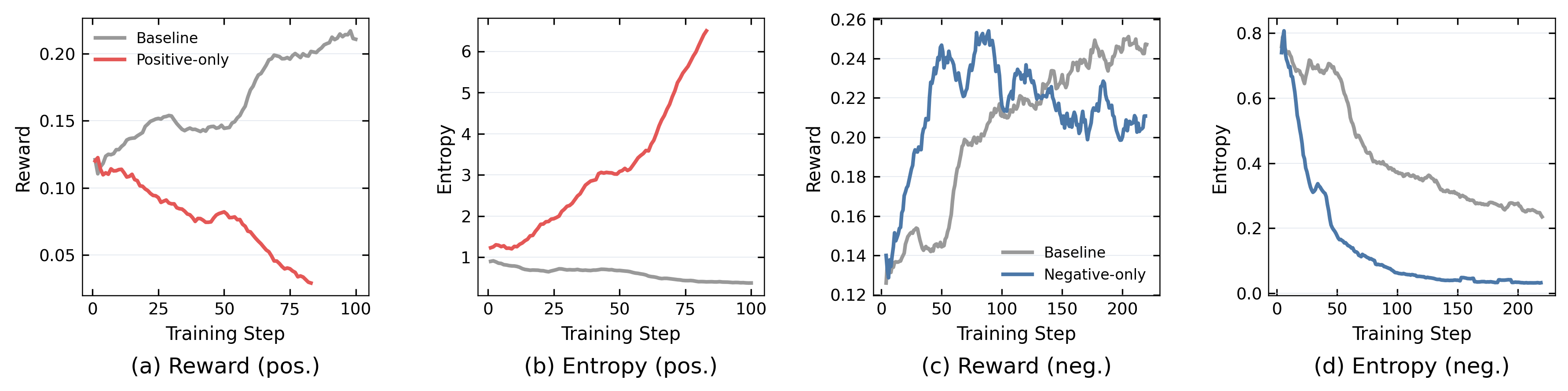}
    \caption{
    \textbf{Ablation study of entropy polarity branches.}
    We retain exclusively positive or negative polarity tokens to compare training dynamics.
    Positive-polarity updates boost entropy but impair reward learning, whereas negative-polarity updates obtain significant reward improvements alongside unstable optimization and declining entropy.
    }
    \label{fig:polarity_directional_control}
    \vspace{-1.5em}
\end{figure}

\section{Method}
\label{sec:method}

\subsection{Motivation: Expansion and Contraction Branches}
\label{sec:empirical_direction}

\paragraph{Positive and negative entropy polarity induce opposite training dynamics.}
To explore the effect of polarity sign in RL, we conduct a single-polarity ablation of RL (Details are presented in Appendix~\ref{app:single_polarity}).
Figure~\ref{fig:polarity_directional_control} shows that retaining only positive polarity updates preserves or expands entropy but weakens reward, whereas retaining only negative polarity updates improves reward early but rapidly contracts entropy and reduces later exploration.
This indicates that entropy polarity separates entropy-expanding and entropy-contracting branches, and that neither branch alone supports stable reinforcement fine-tuning.

\paragraph{Characterizing the functional roles of entropy-polarity branches.}
\begin{wrapfigure}{r}{0.48\linewidth}
    \vspace{-3em}
    \centering
    \includegraphics[width=0.98\linewidth]{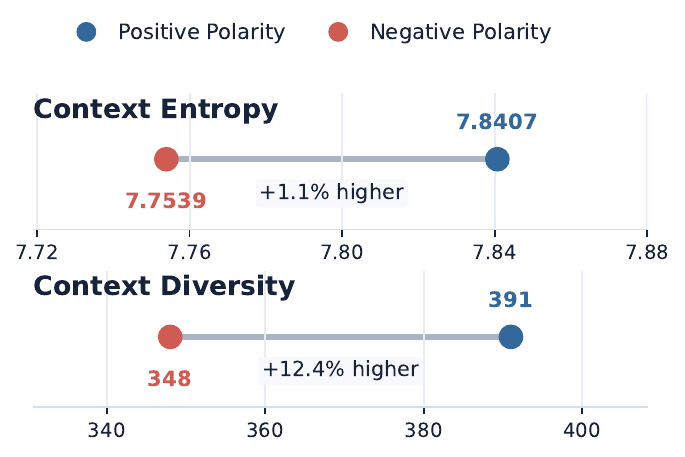}
    \caption{
    Contextual statistics of positive and negative polarity.
    }
    \label{fig:polarity_context_stats}
    \vspace{-1.3em}
\end{wrapfigure}

To understand this distinction, we inspect tokens from each polarity branch.
Quantitatively, Figure~\ref{fig:polarity_context_stats} shows that positive-polarity tokens are associated with both higher context entropy and greater context diversity, indicating that they tend to occur in more open next-token distributions.
We further examine their lexical profiles in Figure~\ref{fig:polarity_wordcloud}.
Positive-polarity tokens cluster around connective and structural tokens that keep continuations open, whereas negative-polarity tokens concentrate on formulaic reasoning fragments and answer-oriented transitions in more constrained contexts.
More discussion about the position concentration of entropy polarity is supplemented in Appendix~\ref{app:polarity_position}.
Together, these observations suggest that positive-polarity updates help preserve action diversity, while negative-polarity updates sharpen reward-aligned exploitation.

\paragraph{Motivation for branch-level control.}
Taken together, these insights indicate complementary roles: \textbf{the positive polarity branch preserves exploratory capacity, whereas the negative polarity branch strengthens exploitation}.
An effective RFT strategy should therefore keep both branches while controlling their relative influence over training, rather than favoring one entropy direction or treating them uniformly, as in GRPO.
Since the desired exploration--exploitation balance changes over time, this control should be adaptive rather than fixed.

\subsection{Polarity-Aware Policy Optimization}
\label{sec:method_papo}

To this end, we propose \textbf{P}olarity-\textbf{A}ware \textbf{P}olicy \textbf{O}ptimization (\methodname{}), which performs branch-level control by reweighting token advantages according to entropy polarity. Rather than targeting a fixed entropy level, \methodname{} uses the empirical entropy trajectory as an online phase signal to redistribute optimization pressure between entropy-expanding and entropy-contracting updates, protecting exploration when entropy declines rapidly and strengthening exploitation once the trajectory stabilizes.

\paragraph{General Formulation.}
We formalize this design by modulating the token advantage. For training step $k$, the polarity-weighted advantage $\tilde{A}_t^{(i)}$ is defined as:
\begin{equation}
\tilde{A}_t^{(i)} =
\begin{cases}
A^{(i)} \cdot \omega_{\mathrm{pos}}(k), & \mathcal{P}(s_t^{(i)},y_t^{(i)},A^{(i)})\geq 0, \\ 
A^{(i)} \cdot \omega_{\mathrm{neg}}(k), & \mathcal{P}(s_t^{(i)},y_t^{(i)},A^{(i)})<0,
\end{cases}
\label{eq:reweight}
\end{equation}
Here $A^{(i)}$ is the group-normalized advantage, while the branch assignment is token-specific through the sign of $\mathcal{P}(s_t^{(i)},y_t^{(i)},A^{(i)})$.
This formulation preserves the original reward objective and recovers GRPO when $\omega_{\mathrm{pos}} = \omega_{\mathrm{neg}} = 1$.

\begin{figure}[t]
    \centering
    \includegraphics[width=1.0\linewidth]{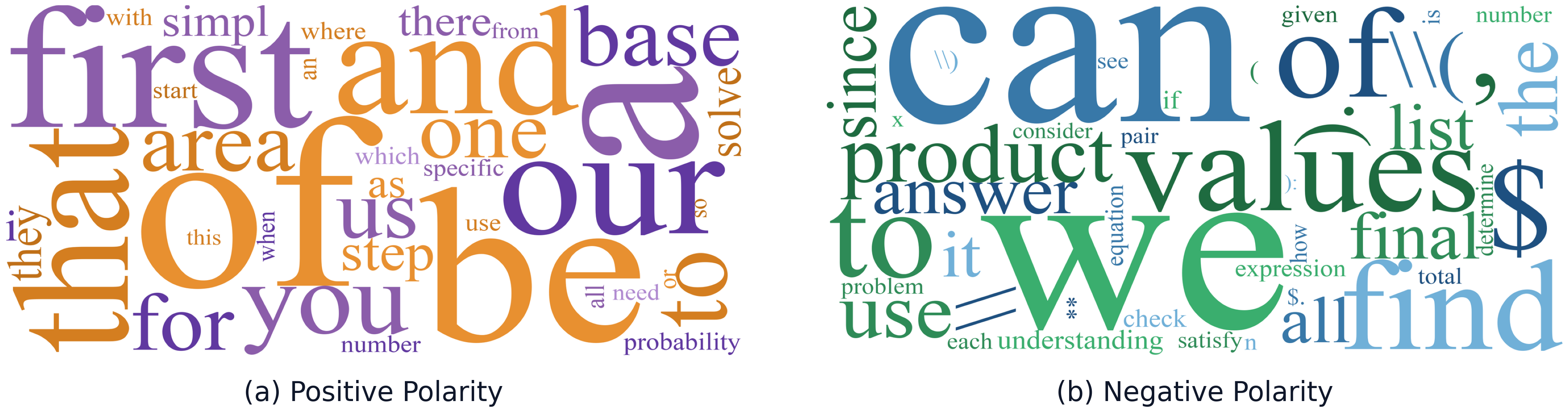}
    \caption{
    Word clouds of positive and negative polarity branches with the high absolute \(\|\mathcal{P}\|\) polarity values, respectively.
    The positive branch (\(\mathcal{P}>0\)) is enriched in connective and structural expressions, e.g., ``and'', ``that'', ``first'', and ``step'', consistent with more open local continuations; the negative branch (\(\mathcal{P}<0\)) instead concentrates on formulaic reasoning and answer-oriented expressions, e.g., ``find'', ``values'', ``product'', and ``final'', consistent with more constrained solution trajectories.
    }
    \vspace{-1.5em}
    \label{fig:polarity_wordcloud}
\end{figure}

\paragraph{Adaptive Entropy Modulation.}

Since the desired exploration--exploitation balance depends on the training dynamics, fixed polarity weights are generally suboptimal.
Let $s_k \triangleq \mathrm{EMA}_{\beta}(\Delta H)$ denote the smoothed entropy slope,
\methodname{} uses the empirical entropy trajectory through $s_k$ as an online phase signal.
After a short warmup with neutral weights, we record the warmup slope and entropy as the reference rate $s_{\mathrm{ref}}<0$ and reference level $h_{\mathrm{ref}}$, respectively.

By rescaling the current slope against $s_{\mathrm{ref}}$, we instantiate a normalized progress metric $p_k$ that measures how much the current entropy decay has recovered from the warmup contraction rate:
\begin{equation}
p_k = \mathrm{clip} \left( \frac{s_k - s_{\mathrm{ref}}}{-s_{\mathrm{ref}} + \varepsilon}, 0, 1 \right).
\label{eq:recovery}
\end{equation}
Smaller $p_k$ indicates ongoing entropy collapse, while larger $p_k$ indicates recovery. We then map $p_k$ to polarity weights through a quadratic rule with reciprocal coupling:
\begin{equation}
\omega_{\mathrm{neg}}(k) = \omega_{\min} + (\omega_{\max} - \omega_{\min}) p_k^2, \qquad
\omega_{\mathrm{pos}}(k) = \frac{1}{\omega_{\mathrm{neg}}(k)}.
\label{eq:omega_logic}
\end{equation}
When $p_k \approx 0$, the controller protects exploration by suppressing entropy-contracting updates and strengthening entropy-expanding ones; as $p_k$ increases, this bias relaxes toward neutral weighting and can mildly favor exploitation.
The quadratic mapping keeps the response conservative, the reciprocal coupling avoids global learning-rate confounding, and a gate at $\gamma_{\mathrm{gate}}\cdot h_{\mathrm{ref}}$ disables reweighting in the extreme low-entropy regime.
The pseudocode of \methodname{} is illustrated in Algorithm~\ref{alg:method}.

\section{Experiments}
\label{sec:exp}
\subsection{Experimental Setup}
\paragraph{Datasets and Models.} 
For the main mathematical reasoning experiments, we train on DAPO-Math-17K~\citep{Yu:2026:DAPO}, a widely adopted high-quality RLVR dataset. 
We use Qwen2.5-7B-Base and Qwen2.5-14B-Base~\citep{qwen2} as base models, ensuring consistency with baseline comparisons.
Additional agentic reasoning experiments use a separate training setup, described in Section~\ref{sec:analysis}.

\paragraph{Evaluation.}

To comprehensively assess the model capabilities, we employ a diverse suite of benchmarks encompassing mathematical reasoning, code generation, and general multi-task proficiency.
Specifically: (1) Mathematical reasoning: OlympiadBench~\citep{He:2024:olympiadbench}, AMC~\citep{Li:2024:AMC}, MATH500~\citep{Lightman:2024:MATH500}, and AIME~\citep{Zhang:2024:AIME}; 
(2) Code generation: CRUX~\citep{Gu:2024:CRUX}; (3) Instruction-following and general abilities: IFEval~\citep{Zhou:2023:IFEval} and MMLU-Pro~\citep{Wang:2024:MMLU}.

\paragraph{Implementation Details.}
We implement our algorithms based on the VeRL framework and conduct experiments on 16 $\times$ H20 GPUs.
During RL training, the learning rate is set to $1\times10^{-6}$, with a global
batch size of 128, and gradient accumulation over 64 steps.
For each group, 8 candidate responses are sampled.
More implementation illustrations are in Appendix~\ref{app:math_config}.

\subsection{Main Results}

\begin{table}[t!]
    \centering
    \caption{\textbf{Performance (\%) on in-domain and out-of-domain benchmarks.} Cells with deeper background color correspond to better performance within each model group.}
    \vspace{-0mm}
    \label{Tble:Main_Exp_PAPO}
    \resizebox{\textwidth}{!}{
    \begin{tabular}{l|cccccc|ccc}
    \toprule[1.6pt]
        & \multicolumn{6}{c|}{\textbf{In-Domain Performance}} 
        & \multicolumn{3}{c}{\textbf{Out-of-Domain}} \\ 
        \cmidrule{2-7} \cmidrule{8-10}
        \textbf{Methods} 
        & \textbf{Olympiad} 
        & \textbf{AMC} 
        & \textbf{AIME24} 
        & \textbf{AIME25} 
        & \textbf{MATH} 
        & \textbf{Minerva} 
        & \textbf{CRUX} 
        & \textbf{IFEval} 
        & \textbf{MMLU-Pro} \\

        \midrule
            \multicolumn{10}{c}{\textit{\textbf{Qwen2.5-7B-Base}}} \\
        \midrule
        - Base 
        & \cellcolor{lightcyan1}{22.30} 
        & \cellcolor{lightcyan1}{31.56} 
        & \cellcolor{lightcyan1}{4.00} 
        & \cellcolor{lightcyan1}{2.00} 
        & \cellcolor{lightcyan1}{50.90} 
        & \cellcolor{lightcyan1}{17.10} 
        & \cellcolor{lightcyan1}{16.25} 
        & \cellcolor{lightcyan1}{27.54} 
        & \cellcolor{lightcyan2}{36.98} \\

        - DAPO 
        & \cellcolor{lightcyan4}{33.12} 
        & \cellcolor{lightcyan3}{49.38} 
        & \cellcolor{lightcyan3}{13.75} 
        & \cellcolor{lightcyan4}{7.50} 
        & \cellcolor{lightcyan2}{65.57} 
        & \cellcolor{lightcyan3}{24.31} 
        & \cellcolor{lightcyan4}{30.63} 
        & \cellcolor{lightcyan3}{28.28} 
        & \cellcolor{lightcyan5}{37.81} \\

        - Ent-Reg 
        & \cellcolor{lightcyan3}{33.08} 
        & \cellcolor{lightcyan4}{50.41} 
        & \cellcolor{lightcyan2}{12.91} 
        & \cellcolor{lightcyan2}{6.70} 
        & \cellcolor{lightcyan5}{66.90} 
        & \cellcolor{lightcyan4}{25.32} 
        & \cellcolor{lightcyan3}{29.75} 
        & \cellcolor{lightcyan1}{27.54} 
        & \cellcolor{lightcyan3}{37.02} \\

        - 80/20 
        & \cellcolor{lightcyan2}{32.69} 
        & \cellcolor{lightcyan2}{49.06} 
        & \cellcolor{lightcyan5}{16.25} 
        & \cellcolor{lightcyan3}{7.08} 
        & \cellcolor{lightcyan3}{65.67} 
        & \cellcolor{lightcyan2}{24.13} 
        & \cellcolor{lightcyan2}{28.62} 
        & \cellcolor{lightcyan5}{29.94} 
        & \cellcolor{lightcyan1}{36.92} \\

        - PAPO 
        & \cellcolor{lightcyan5}{34.70} 
        & \cellcolor{lightcyan5}{51.56} 
        & \cellcolor{lightcyan4}{14.58} 
        & \cellcolor{lightcyan5}{8.30} 
        & \cellcolor{lightcyan4}{65.90} 
        & \cellcolor{lightcyan5}{26.75} 
        & \cellcolor{lightcyan5}{30.88} 
        & \cellcolor{lightcyan4}{29.57} 
        & \cellcolor{lightcyan5}{37.81} \\

        \specialrule{1pt}{0.4ex}{0.4ex}
            \multicolumn{10}{c}{\textit{\textbf{Qwen2.5-14B-Base}}} \\
        \midrule
        - Base 
        & \cellcolor{lightred1}{26.04} 
        & \cellcolor{lightred1}{36.56} 
        & \cellcolor{lightred1}{6.67} 
        & \cellcolor{lightred1}{5.00} 
        & \cellcolor{lightred1}{56.23} 
        & \cellcolor{lightred1}{20.77} 
        & \cellcolor{lightred1}{28.38} 
        & \cellcolor{lightred1}{30.31} 
        & \cellcolor{lightred1}{47.33} \\

        - DAPO 
        & \cellcolor{lightred4}{38.57} 
        & \cellcolor{lightred4}{60.63} 
        & \cellcolor{lightred2}{12.50} 
        & \cellcolor{lightred2}{11.25} 
        & \cellcolor{lightred3}{70.03} 
        & \cellcolor{lightred4}{30.93} 
        & \cellcolor{lightred3}{46.38} 
        & \cellcolor{lightred2}{33.27} 
        & \cellcolor{lightred3}{48.00} \\

        - Ent-Reg 
        & \cellcolor{lightred3}{38.34} 
        & \cellcolor{lightred2}{58.75} 
        & \cellcolor{lightred4}{13.75} 
        & \cellcolor{lightred3}{11.67} 
        & \cellcolor{lightred5}{71.40} 
        & \cellcolor{lightred2}{30.83} 
        & \cellcolor{lightred2}{46.25} 
        & \cellcolor{lightred3}{33.46} 
        & \cellcolor{lightred4}{48.15} \\

        - 80/20 
        & \cellcolor{lightred2}{37.59} 
        & \cellcolor{lightred3}{59.38} 
        & \cellcolor{lightred3}{13.33} 
        & \cellcolor{lightred4}{12.08} 
        & \cellcolor{lightred2}{69.38} 
        & \cellcolor{lightred3}{30.88} 
        & \cellcolor{lightred4}{47.12} 
        & \cellcolor{lightred4}{33.64} 
        & \cellcolor{lightred5}{48.25} \\

        - PAPO 
        & \cellcolor{lightred5}{39.22} 
        & \cellcolor{lightred5}{62.50} 
        & \cellcolor{lightred5}{15.83} 
        & \cellcolor{lightred5}{13.33} 
        & \cellcolor{lightred4}{71.38} 
        & \cellcolor{lightred5}{31.94} 
        & \cellcolor{lightred5}{48.25} 
        & \cellcolor{lightred5}{33.83} 
        & \cellcolor{lightred2}{47.89} \\

    \bottomrule[1.6pt]
    \end{tabular}}
    \vspace{-3mm}
\end{table}

\begin{figure*}[t]
    \centering
    \includegraphics[width=\linewidth]{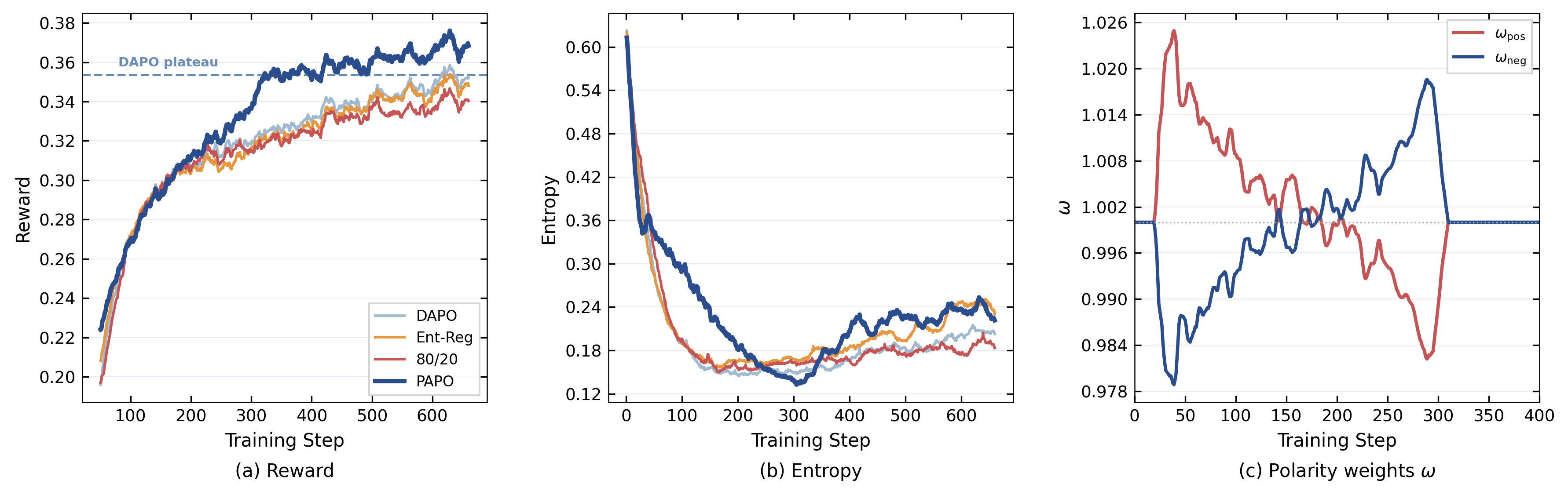}
    \caption{\textbf{Training dynamics of PAPO.} 
    Reward~(left), entropy~(middle) and polarity weight~(right) of Qwen2.5-14B-Base trained on DAPO-Math-17k. 
    \methodname{} delivers significant reward improvement while maintaining a healthier entropy dynamic, and its adaptive entropy modulation relaxes back to near-neutral weights after the transient collapse-recovery phase.
    }
    \label{fig:math_reward_entropy}
\end{figure*}

\paragraph{Superior performance on mathematical reasoning.}
Table~\ref{Tble:Main_Exp_PAPO} reports the results of mathematical reasoning on six benchmarks.
\methodname{} achieves state-of-the-art performance on most benchmarks and delivers consistent overall improvements across both 7B and 14B model scales:
For Qwen2.5-14B, \methodname{} outperforms baselines by 3.3\% on AIME24 and 2.1\% on AIME25; it obtains clear gains of 2.4\% on Minerva and 2.2\% on AMC on Qwen2.5-7B.
Entropy regularization works well on simple benchmarks like MATH500 yet underperforms on challenging competition benchmarks, revealing the limitation of global uniform regularization.
Thus \methodname{}'s token-level polarity-aware reweighting provides finer-grained control, yielding more consistent improvements across benchmarks and scales.

\paragraph{PAPO reshapes policy training through phase-adaptive control.}
Figure~\ref{fig:math_reward_entropy} shows that PAPO regulates optimization by adaptively rebalancing exploration and exploitation across training phases.
In the early stage, the controller prioritizes exploration protection by down-weighting the entropy-contracting branch (\(\omega_{\mathrm{neg}}<1\)), which helps avoid premature contraction.
As entropy decay slows, the controller gradually restores and then mildly favors exploitation (\(\omega_{\mathrm{neg}}>1\)).
After the late-stage gate is triggered, the reweighting is deactivated, and training proceeds with neutral branch weights.
This phase-wise control produces a non-monotonic entropy trajectory: \methodname{} exhibits a transient mid-training contraction, then stabilizes at a higher reward and entropy level than DAPO.

\paragraph{\methodname{} achieves significant training efficiency.}
The reward curve in Figure~\ref{fig:math_reward_entropy} displays significant training efficiency: our method enables the policy model to reach the mature reward performance of DAPO with merely half the training budget, and continues to increase throughout the training stage. 
Such rapid capability advancement does not stem from over-aggressive early exploitation. As validated by entropy trajectories, the policy preserves strong exploratory capacity, manifesting that the rapid performance growth of our model does not compromise reasoning diversity.

\subsection{Analysis}
\label{sec:analysis}

\paragraph{Generalization on out-of-domain benchmarks.}
We further evaluate the cross-domain generalization of math-optimized checkpoints, including code reasoning (CRUX), instruction following (IFEval), and general knowledge (MMLU-Pro).
Despite being trained exclusively on math, \methodname{} achieves comparable or even superior results over DAPO.
For the 14B variant, our method obtains a prominent performance gain of 1.9\% on CRUX, verifying that reasoning competencies cultivated via mathematical training can be effectively transferred to code-related understanding tasks.
In the 7B setting, \methodname{} yields a 1.3\% improvement on IFEval while maintaining comparable on MMLU-Pro.
Our approach achieves desirable improvements without degrading generalization.

\begin{figure}[t]
    \centering
    \includegraphics[width=0.9\linewidth]{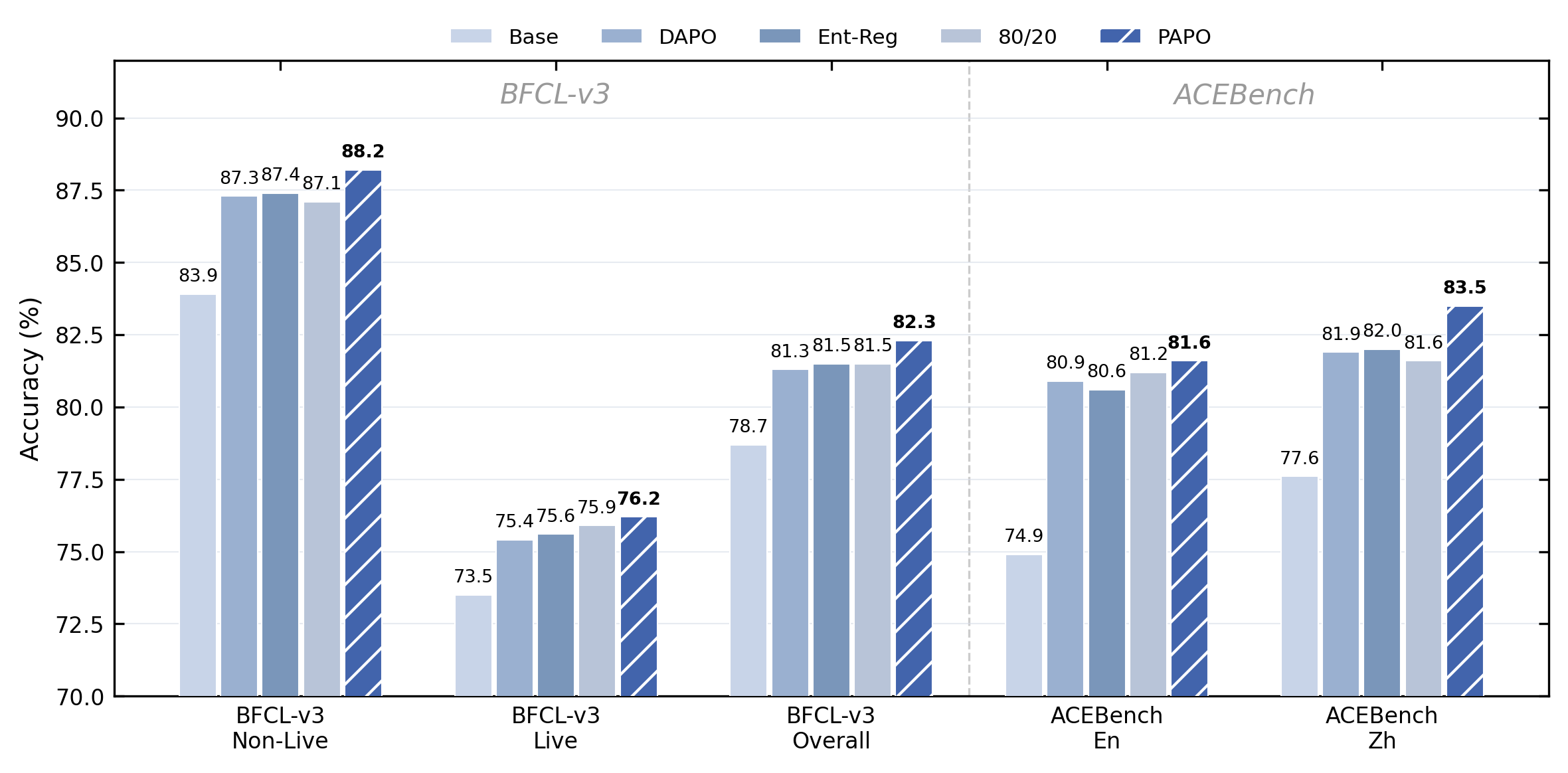}
    \caption{
    \textbf{Performance (\%) on agentic tool-using benchmarks.} 
    \methodname{} achieves the strongest overall performance. It consistently surpasses baselines across different evaluation settings, covering non-live and live scenarios in BFCL, as well as English and Chinese tasks in ACEBench.
    }
  \vspace{1em}
    \label{fig:tool_call}
\end{figure}

\paragraph{Transferability in agentic tool-calling scenarios.}
We further evaluate \methodname{} in agentic tool-call generation scenarios, 
which place high demands on agent-level capabilities such as precise tool intent comprehension and accurate function invocation decision-making, differing from open-ended mathematical reasoning.
This setting tests whether polarity-aware reweighting captures a transferable optimization signal. We train Qwen2.5-14B-Instruct on Tool-N1~\citep{zhang:2025:ToolN1} and evaluate on BFCL-v3~\citep{Patil:2025:BFCL} and ACEBench~\citep{Chen:2025:ACEBench}. As shown in Figure~\ref{fig:tool_call}, \methodname{} achieves the strongest overall performance among matched methods on both benchmarks, achieving an average +1.0\% improvement on BFCL-v3 and the best performance on ACEBench among the compared methods. This suggests that effective-polarity reweighting transfers beyond mathematical reasoning to structured tool-call actions.

\begin{wrapfigure}{r}{0.52\linewidth}
  \centering
  \includegraphics[width=\linewidth]{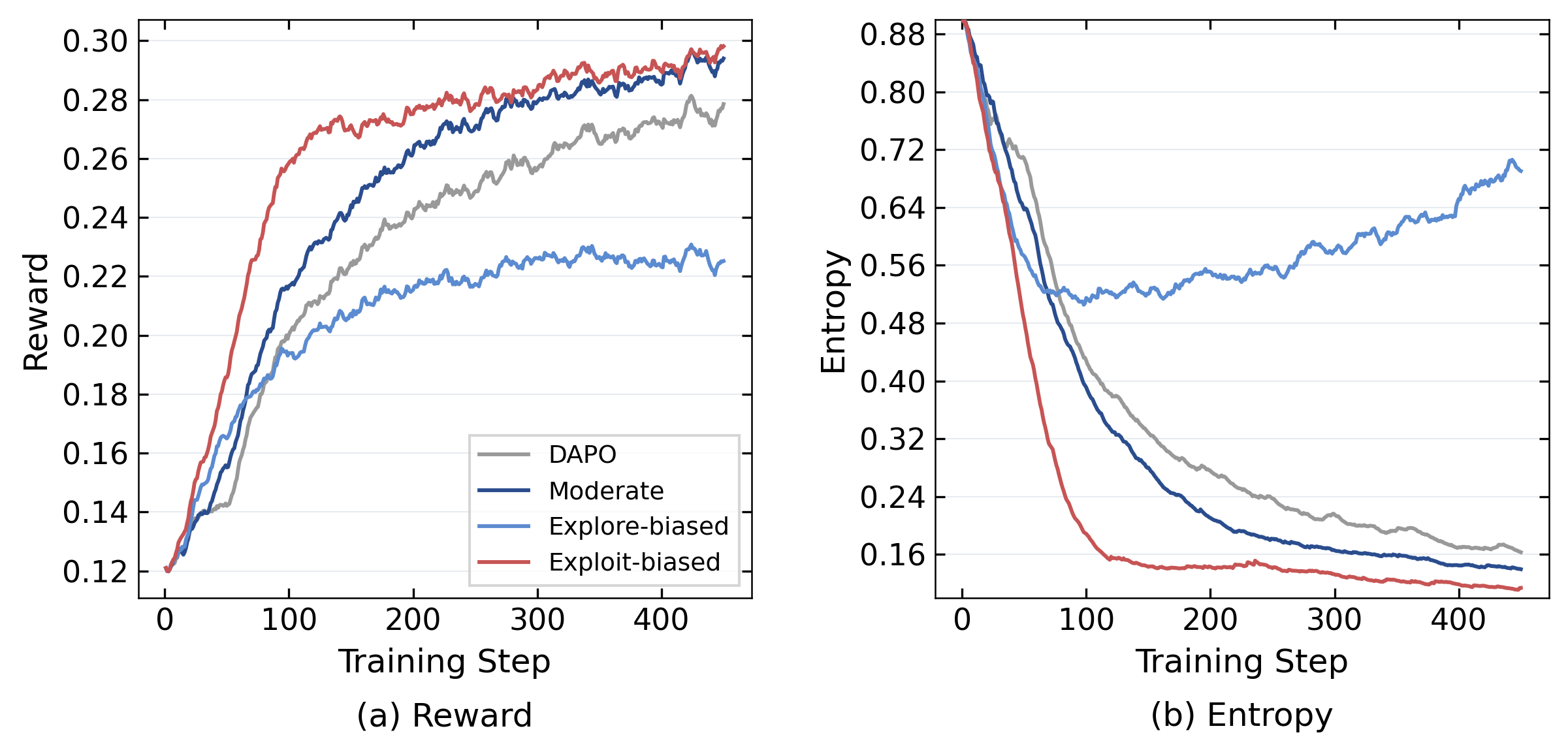}
\caption{Effect of polarity-control strength \(\omega\) on policy dynamics.}
  \label{fig:ablation}
  \vspace{-1.2em}
\end{wrapfigure}

\paragraph{Effect of polarity-aware strength on policy dynamics.}
Figure~\ref{fig:ablation} compares training dynamics with different polarity-control strengths; detailed settings are deferred to Appendix~\ref{app:strength_ablation}. The exploration-oriented setting preserves entropy more aggressively by suppressing entropy-contracting updates, but consequently weakens the optimization pressure needed for reward improvement, leading to slower gains and a lower final reward. The exploitation-oriented setting exhibits the opposite behavior: reward improves rapidly early on because the contractive branch is less restrained, but this acceleration is achieved at the expense of entropy, inducing premature collapse of exploratory capacity and weaker later-stage improvement. 
The moderate setting maintains a controlled entropy decline with partial recovery while eventually reaching a competitive reward level.
In short, polarity control needs to balance the entropy cost of reward gains, keeping the contractive branch sufficiently robust to sustain learning and avoid premature entropy collapse.

\section{Related Work}
\paragraph{Reinforcement Learning with Verifiable Rewards.}
RLVR has become a dominant paradigm for improving the reasoning capabilities of large language models~\citep{Zhang:2025:ReasoningSurvey,Shao:2024:GRPO,zhang2025reinforcement}, motivating a broad line of work on stable policy optimization. GRPO~\citep{Shao:2024:GRPO} pioneered critic-free optimization for reasoning LLMs, while later methods such as DAPO~\citep{Yu:2026:DAPO} and VAPO~\citep{Yu:2026:DAPO} further improved training stability through refined optimization and value modeling. Subsequent work has explored complementary directions, including scalable training recipes~\citep{An:2025:Polaris,He:2025:justrl}, sequence-level stabilization for MoE models~\citep{Zheng:2025:GSPO,Gao:2025:SAPO}, and enhanced credit assignment or self-improving training schemes~\citep{Teng:2026:SKPO,Bai:2026:TTVS}.

\paragraph{Entropy Dynamics and Control.}
Entropy has emerged as a useful lens for understanding exploratory capacity and reasoning behavior in RLVR~\citep{Wang:2025:8020,Cui:2025:EntropyMechanism,Meng:2026:RLSparse,Huang:2026:DirectionofRLVR,Xi:2026:BAPO}. Prior work studies entropy dynamics from both analytical and empirical perspectives, showing that RLVR updates often concentrate on a small set of high-impact reasoning tokens~\citep{Wang:2025:8020,Meng:2026:RLSparse,Huang:2026:DirectionofRLVR}. To mitigate entropy collapse, existing methods mainly rely on global or heuristic interventions, including clipping-based trust-region modifications~\citep{Yu:2026:DAPO,Su:2025:CEGRPO,Chen:2026:ClipControl} and advantage-shaping or adaptive-control schemes~\citep{Chen:2025:EnrtopyAdvantage,Petrenko:2026:EntropyPreserveRL}. In contrast, our work focuses on how sampled token updates directionally expand or contract entropy, and uses this token-level mechanism to implement entropy control through advantage reweighting.

\section{Conclusion, Limitation and Future Work}
\label{sec:conclusion}

\paragraph{Conclusion.}
This work develops a token-level perspective on entropy mechanics in RLVR and introduces intrinsic entropy tendency together with entropy polarity to characterize the directional entropy effect of sampled updates. 
The resulting framework reveals a structural asymmetry and in turn motivates \methodname{}, which adaptively modulates the relative influence of the two polarity branches according to the empirical entropy trajectory.  
Experiments suggest that \methodname{} provides a practical way to improve the exploration–exploitation trade-off in RLVR.

\paragraph{Limitation and Future Work.}
This work is restricted to standard RLVR settings with 7B/14B text-LLMs and verifiable reward formulations, and does not yet address long-horizon regimes involving hundreds of sequential reasoning or interaction steps. An important direction for future work is to examine whether the polarity perspective extends to larger-scale models, multimodal scenarios, longer-horizon decision processes, and broader reward formulations.

\bibliographystyle{plainnat}
\bibliography{references}

\appendix

\section{Theoretical Derivations}
\label{app:full_derivations}

This appendix provides the full derivations for
Proposition~\ref{prop:sampled_logit_update} and
Theorem~\ref{thm:sampled_entropy_change}, and further analyzes the sign and magnitude structure of the intrinsic entropy tendency \(\mathcal{T}\).
Throughout this appendix, we fix a decoding state \(s_t\), and write
\[
p_v := \pi_\theta(v\mid s_t),
\qquad
p_t := \pi_\theta(y_t\mid s_t),
\qquad
\mathcal{H}_t :=
-\sum_{v\in\mathcal{V}}p_v\log p_v .
\]
All quantities are evaluated at the pre-update policy unless otherwise specified.

\subsection{Proof of the Sampled Logit Update}
\label{app:proof_sampled_logit_update}

\paragraph{Restatement.}
For a fixed decoding state \(s_t\), suppose token \(y_t\) is sampled and the local policy is updated by the single policy-gradient term
\(A\log\pi_\theta(y_t\mid s_t)\) with step size \(\eta\).
Then the first-order logit update satisfies
\[
z'_{s_t,v}
=
z_{s_t,v}
+
\eta A\left(\mathbb{I}[v=y_t]-p_v\right),
\qquad
\forall v\in\mathcal{V}.
\]

\paragraph{Proof.}
Under the local softmax parameterization,
\[
p_v
=
\pi_\theta(v\mid s_t)
=
\frac{\exp(z_{s_t,v})}
{\sum_{u\in\mathcal{V}}\exp(z_{s_t,u})}.
\]
For the sampled token \(y_t\), we have
\[
\log\pi_\theta(y_t\mid s_t)
=
z_{s_t,y_t}
-
\log
\left(
\sum_{u\in\mathcal{V}}\exp(z_{s_t,u})
\right).
\]
Differentiating with respect to \(z_{s_t,v}\) gives
\[
\frac{\partial\log\pi_\theta(y_t\mid s_t)}
{\partial z_{s_t,v}}
=
\mathbb{I}[v=y_t]-p_v.
\]
Therefore,
\[
\frac{\partial}
{\partial z_{s_t,v}}
\left(
A\log\pi_\theta(y_t\mid s_t)
\right)
=
A\left(\mathbb{I}[v=y_t]-p_v\right).
\]
Applying one gradient-ascent step yields
\[
z'_{s_t,v}
=
z_{s_t,v}
+
\eta A\left(\mathbb{I}[v=y_t]-p_v\right),
\]
which proves Proposition~\ref{prop:sampled_logit_update}.
\qed

\subsection{Proof of the Entropy-Polarity Theorem}
\label{app:proof_main_theorem}

\paragraph{Restatement.}
Under the same local softmax setting, the one-step entropy change induced by the sampled update in Proposition~\ref{prop:sampled_logit_update} satisfies
\[
\Delta\mathcal{H}_t
=
-\eta A\,t_1(s_t,y_t)
+
\eta A\,t_2(s_t)
+
O(\eta^2),
\]
where
\[
t_1(s_t,y_t)
:=
p_t(\mathcal{H}_t+\log p_t),
\qquad
t_2(s_t)
:=
\sum_{v\in\mathcal{V}}
p_v^2(\mathcal{H}_t+\log p_v).
\]
Equivalently, with
\[
\mathcal{T}(s_t,y_t)
:=
-t_1(s_t,y_t)+t_2(s_t),
\]
we have
\[
\Delta\mathcal{H}_t
=
\eta A\,\mathcal{T}(s_t,y_t)
+
O(\eta^2).
\]

\paragraph{Proof.}
We first compute the gradient of the local entropy with respect to the logits.
Since
\[
\mathcal{H}_t
=
-\sum_{u\in\mathcal{V}}p_u\log p_u,
\]
and the softmax derivative satisfies
\[
\frac{\partial p_u}{\partial z_{s_t,v}}
=
p_u(\mathbb{I}[u=v]-p_v),
\]
we obtain
\begin{align}
\frac{\partial\mathcal{H}_t}{\partial z_{s_t,v}}
&=
-\sum_{u\in\mathcal{V}}
\frac{\partial p_u}{\partial z_{s_t,v}}
(\log p_u+1)
\nonumber\\
&=
-\sum_{u\in\mathcal{V}}
p_u(\mathbb{I}[u=v]-p_v)(\log p_u+1)
\nonumber\\
&=
-p_v(\log p_v+1)
+
p_v\sum_{u\in\mathcal{V}}p_u(\log p_u+1).
\end{align}
Using
\[
\sum_{u\in\mathcal{V}}p_u(\log p_u+1)
=
-\mathcal{H}_t+1,
\]
we have
\[
\frac{\partial\mathcal{H}_t}{\partial z_{s_t,v}}
=
-p_v(\mathcal{H}_t+\log p_v).
\]

Let
\[
\delta z_{s_t,v}:=z'_{s_t,v}-z_{s_t,v}.
\]
By the first-order Taylor expansion of entropy around the pre-update logits,
\[
\Delta\mathcal{H}_t
=
\sum_{v\in\mathcal{V}}
\frac{\partial\mathcal{H}_t}{\partial z_{s_t,v}}
\delta z_{s_t,v}
+
O(\|\delta z_t\|^2).
\]
From Proposition~\ref{prop:sampled_logit_update},
\[
\delta z_{s_t,v}
=
\eta A(\mathbb{I}[v=y_t]-p_v).
\]
Substituting the entropy gradient and the logit update gives
\begin{align}
\Delta\mathcal{H}_t
&=
-\eta A
\sum_{v\in\mathcal{V}}
p_v(\mathcal{H}_t+\log p_v)
(\mathbb{I}[v=y_t]-p_v)
+
O(\eta^2)
\nonumber\\
&=
-\eta A
p_t(\mathcal{H}_t+\log p_t)
+
\eta A
\sum_{v\in\mathcal{V}}
p_v^2(\mathcal{H}_t+\log p_v)
+
O(\eta^2).
\end{align}
By the definitions in Eq.~\eqref{eq:t1_t2_def}, this becomes
\[
\Delta\mathcal{H}_t
=
-\eta A\,t_1(s_t,y_t)
+
\eta A\,t_2(s_t)
+
O(\eta^2),
\]
which proves Eq.~\eqref{eq:entropy_decomposition}.

Using Eq.~\eqref{eq:structural_polarity},
\[
\mathcal{T}(s_t,y_t)
=
-t_1(s_t,y_t)+t_2(s_t),
\]
we further obtain
\[
\Delta\mathcal{H}_t
=
\eta A\,\mathcal{T}(s_t,y_t)
+
O(\eta^2).
\]
Together with the definition
\[
\mathcal{P}(s_t,y_t,A)=A\mathcal{T}(s_t,y_t)
\]
in Eq.~\eqref{eq:polarity_and_entropy_change}, this also gives
\[
\Delta\mathcal{H}_t
=
\eta\,\mathcal{P}(s_t,y_t,A)
+
O(\eta^2).
\]
This completes the proof of Theorem~\ref{thm:sampled_entropy_change}.
\qed

\subsection{Sign Analysis and Directional Asymmetry of Entropy Tendency}
\label{app:polarity_analysis}

Theorem~\ref{thm:sampled_entropy_change} shows that the intrinsic entropy tendency is
\[
\mathcal{T}(s_t,y_t)
=
-t_1(s_t,y_t)+t_2(s_t).
\]
This section further analyzes how \(t_1\) and \(t_2\) determine the sign and magnitude of \(\mathcal{T}\).

\paragraph{The sampled-token term \(t_1\).}
The term
\[
t_1(s_t,y_t)
=
p_t(\mathcal{H}_t+\log p_t)
\]
is the only component in \(\mathcal{T}\) that depends on the sampled token \(y_t\).
Since \(p_t>0\), its sign is fully determined by \(\mathcal{H}_t+\log p_t\):
\[
\operatorname{sgn}\!\left(t_1(s_t,y_t)\right)
=
\operatorname{sgn}\!\left(\mathcal{H}_t+\log p_t\right).
\]
Therefore,
\[
t_1(s_t,y_t)>0
\quad\Longleftrightarrow\quad
p_t>\exp(-\mathcal{H}_t),
\]
and
\[
t_1(s_t,y_t)<0
\quad\Longleftrightarrow\quad
p_t<\exp(-\mathcal{H}_t).
\]
Thus, \(t_1\) separates sampled tokens by an entropy-induced probability threshold. Tokens above this threshold have a positive \(t_1\), and hence contribute negatively to \(\mathcal{T}\) through the term \(-t_1\). Tokens below this threshold have a negative \(t_1\), and hence contribute positively to \(\mathcal{T}\).

\paragraph{The distribution term \(t_2\).}
The term
\[
t_2(s_t)
=
\sum_{v\in\mathcal{V}}
p_v^2(\mathcal{H}_t+\log p_v)
\]
does not depend on the sampled token. Instead, it is shared by all sampled-token updates at the same decoding state and depends only on the shape of the local next-token distribution.
A useful equivalent form shows that \(t_2\) is always non-negative:
\[
\begin{aligned}
t_2(s_t)
&=
\sum_{v\in\mathcal{V}}p_v^2\log p_v
+
\mathcal{H}_t\sum_{v\in\mathcal{V}}p_v^2  \\
&=
\mathbb{E}_{v\sim p}\!\left[p_v\log p_v\right]
-
\mathbb{E}_{v\sim p}\!\left[p_v\right]
\mathbb{E}_{v\sim p}\!\left[\log p_v\right]  \\
&=
\operatorname{Cov}_{v\sim p}\!\left(p_v,\log p_v\right)
\ge 0 .
\end{aligned}
\]
The last inequality follows because \(p_v\) and \(\log p_v\) are monotone functions of the same probability value.
Therefore, \(t_2(s_t)\) acts as a non-negative state-wise distributional correction.
Since it enters \(\mathcal{T}(s_t,y_t)=-t_1(s_t,y_t)+t_2(s_t)\) with a positive sign, it shifts the intrinsic tendency toward entropy expansion for all sampled tokens at the same state.

\paragraph{Sign analysis of \(\mathcal{T}\).}
Combining the two terms, the sign of \(\mathcal{T}\) is determined by the competition between the sampled-token term \(t_1\) and the distribution term \(t_2\):
\[
\mathcal{T}(s_t,y_t)<0
\quad\Longleftrightarrow\quad
t_1(s_t,y_t)>t_2(s_t),
\]
and
\[
\mathcal{T}(s_t,y_t)>0
\quad\Longleftrightarrow\quad
t_1(s_t,y_t)<t_2(s_t).
\]
Therefore, when the sampled-token contraction effect \(t_1\) dominates the distributional correction \(t_2\), the token has a negative intrinsic tendency and tends to decrease local entropy under positive reinforcement. Conversely, when \(t_2\) dominates \(t_1\), or when \(t_1<0\), the token has a positive intrinsic tendency and tends to increase local entropy under positive reinforcement. For negative-advantage updates, the realized direction is reversed by the sign of \(A\), which motivates the effective entropy polarity
\[
\mathcal{P}(s_t,y_t,A)=A\mathcal{T}(s_t,y_t).
\]

\paragraph{Pointwise asymmetry and realized-scale comparability.}
The decomposition reveals a subtle distinction between the pointwise form of the intrinsic terms and the empirical scale of the realized components that appear in the entropy change.
For a fixed sampled token, the sampled-token term has magnitude
\[
|t_1(s_t,y_t)|
=
p_t|\mathcal{H}_t+\log p_t|,
\]
which is pointwise first order in the sampled-token probability \(p_t\), up to the logarithmic factor.
The distribution term is
\[
t_2(s_t)
=
\sum_{v\in\mathcal{V}}
p_v^2(\mathcal{H}_t+\log p_v),
\]
and therefore aggregates squared-probability contributions over the whole next-token distribution.

This pointwise first-order versus second-order form should not be interpreted as implying that the realized distributional component is negligible in practice.
Indeed, under the same policy that generates the sampled token, and for a fixed signed training signal \(A\),
\[
\mathbb{E}_{y_t\sim\pi(\cdot\mid s_t)}
\left[
A\,t_1(s_t,y_t)
\right]
=
A
\sum_{v\in\mathcal{V}}
p_v\,p_v(\mathcal{H}_t+\log p_v)
=
A\,t_2(s_t).
\]
Thus, when rollout statistics are collected over sampled token instances, the sampling process itself gives the sampled-token component an additional probability weighting, making the empirical scale of \(A t_1\) naturally comparable to that of \(A t_2\).
This explains why the two realized components can appear on the same order in practice, despite the different pointwise dependence of \(t_1\) and \(t_2\) on token probability.

The resulting asymmetry is therefore directional rather than a universal magnitude dominance.
High-probability tokens are sampled more frequently and fall on the contractive side of \(-t_1\); when this token-specific concentration effect exceeds the positive correction \(t_2\), the intrinsic tendency becomes contractive.
In contrast, entropy expansion is associated with lower-probability samples or cases where the distributional correction is strong enough to offset the sampled-token concentration effect.

\section{Experimental Details}
\label{sec:appendix_exp_details}

\subsection{Pseudo Code of PAPO}

To intuitively present the pipeline of PAPO, we summarize the pseudo codes of PAPO in Algorithm~\ref{alg:papo}.

\begin{algorithm}[t]
\caption{\methodname{}: Polarity-Aware Policy Optimization}
\label{alg:papo}%
\label{alg:method}%
\begin{algorithmic}[1]
\Require Policy $\pi_\theta$; EMA coefficient $\beta$; bounds $\omega_{\min}, \omega_{\max}$; warmup $N_{\mathrm{w}}$; gate ratio $\gamma_{\mathrm{gate}}$
\State $s_0 \leftarrow 0$;\; $s_{\mathrm{ref}}, h_{\mathrm{ref}} \leftarrow \texttt{None}$;\; \texttt{active} $\leftarrow$ \texttt{True}
\For{step $k = 1, 2, \dots$}
  \State Sample responses; compute advantages $A^{(i)}$
  \ForAll{response $i$, position $t$} \Comment{\textcolor{gray}{Polarity}}
    \State $\mu \leftarrow \sum_{v} \pi(v \mid s_t) \log \pi(v \mid s_t)$;\;\;
           $t_1 \leftarrow \pi(y_t \mid s_t)(\log \pi(y_t \mid s_t) - \mu)$
    \State $t_2 \leftarrow \sum_{v} \pi(v \mid s_t)^2 (\log \pi(v \mid s_t) - \mu)$;\;\;
           $\mathcal{P}_t^{(i)} \leftarrow A^{(i)} \cdot (-t_1 + t_2)$
  \EndFor
  \State $h_k \leftarrow$ mean entropy;\; $s_k \leftarrow \beta\, s_{k-1} + (1{-}\beta)(h_k - h_{k-1})$
  \If{$k = N_{\mathrm{w}}$} $s_{\mathrm{ref}} \leftarrow s_k$;\; $h_{\mathrm{ref}} \leftarrow h_k$ \Comment{Lock references}
  \EndIf
  \If{\texttt{active} \textbf{and} $\mathrm{EMA}(h_k) < \gamma_{\mathrm{gate}} h_{\mathrm{ref}}$} \texttt{active} $\leftarrow$ \texttt{False} \Comment{Deactivate}
  \EndIf
  \If{$k > N_{\mathrm{w}}$ \textbf{and} \texttt{active}} \Comment{\textcolor{gray}{Adaptive weights}}
    \State $p_k \leftarrow \mathrm{clip}_{[0,1]}\!\bigl(\frac{s_k - s_{\mathrm{ref}}}{-s_{\mathrm{ref}}}\bigr)$
    \State $\omega_{\mathrm{neg}} \leftarrow \omega_{\min} + (\omega_{\max} - \omega_{\min})\,p_k^2$;\;\; $\omega_{\mathrm{pos}} \leftarrow 1/\omega_{\mathrm{neg}}$
  \Else\;\; $\omega_{\mathrm{pos}}, \omega_{\mathrm{neg}} \leftarrow 1$
  \EndIf
  \ForAll{response $i$, position $t$} \Comment{\textcolor{gray}{Reweight (Eq.~\eqref{eq:reweight})}}
    \State $\tilde{A}_t^{(i)} \leftarrow A^{(i)} \cdot \omega_{\mathrm{pos}}$ if $\mathcal{P}_t^{(i)} {>} 0$,\; $A^{(i)} \cdot \omega_{\mathrm{neg}}$ if $\mathcal{P}_t^{(i)} {<} 0$,\; $A^{(i)}$ otherwise
  \EndFor
  \State Update $\theta$ via clipped policy gradient with $\tilde{A}_t^{(i)}$
\EndFor
\end{algorithmic}
\end{algorithm}

\subsection{Evaluation Details}

\label{app:eval_details}

We evaluate across benchmarks spanning mathematical reasoning, code generation, instruction following, and general multi-task abilities.

\paragraph{Mathematical reasoning.}
We evaluate on six benchmarks: OlympiadBench~\citep{He:2024:olympiadbench}, AMC23~\citep{Li:2024:AMC}, AIME24, AIME25~\citep{Zhang:2024:AIME}, MATH500~\citep{Lightman:2024:MATH500}, and Minerva.
For all math benchmarks, answers are extracted from model outputs via rule-based parsing and compared against ground-truth solutions using exact match.
In the main table (Table~\ref{Tble:Main_Exp_PAPO}), we report mean@8 accuracy: each problem is solved 8 times independently and the average accuracy across all attempts is reported.

\paragraph{Out-of-domain benchmarks.}
To assess generalization beyond the math training distribution, we evaluate on three out-of-domain benchmarks:
CRUX~\citep{Gu:2024:CRUX} for code reasoning, IFEval~\citep{Zhou:2023:IFEval} for instruction following, and MMLU-Pro~\citep{Wang:2024:MMLU} for general knowledge.
CRUX is evaluated via the ZeroEval library; IFEval and MMLU-Pro are evaluated using lm-evaluation-harness.
For all OOD benchmarks, we report pass@1 accuracy.

\paragraph{Agentic benchmarks.}
We evaluate on BFCL-v3~\citep{Patil:2025:BFCL} and ACEBench~\citep{Chen:2025:ACEBench}, reporting accuracy for all results. For BFCL, we follow the official benchmark setup and evaluate on the Live and Non-live subsets while excluding the Multi-turn split. Non-live consists of synthetically generated or curated test cases, whereas Live contains real user-contributed queries. Each subset includes four categories: Simple, Multiple, Parallel, and Parallel Multiple. 
For ACEBench, we exclude multi-turn cases and focus on two subsets under the Normal category: Atom and Single-turn.

\subsection{Baselines}
\label{app:baseline_details}

\paragraph{DAPO.}
DAPO~\citep{Yu:2026:DAPO} serves as the vanilla RLVR baseline.
It uses the clip-higher mechanism ($\epsilon_{\mathrm{high}} > \epsilon_{\mathrm{low}}$) to widen the trust region for advantaged tokens, and employs dynamic sampling: for each prompt, more candidate responses are generated than the training batch requires, and prompts whose responses are all correct or all incorrect are filtered out so that only prompts with mixed outcomes contribute to the gradient.
DAPO applies no explicit entropy control beyond these mechanisms.

\paragraph{Entropy Regularization.}
This baseline augments the GRPO objective with an entropy bonus term $\alpha \mathcal{H}(\pi_\theta)$ added to the loss, encouraging the policy to maintain higher entropy throughout training.
The regularization is applied uniformly across all tokens regardless of their role in entropy dynamics.

\paragraph{80/20.}
Following~\cite{Wang:2025:8020}, this method retains only the top 20\% of tokens ranked by per-token entropy and computes the policy gradient exclusively on this high-entropy subset.
The motivation is that reasoning-critical tokens concentrate at high-entropy positions; by discarding low-entropy tokens, the gradient signal focuses on the subset most relevant to exploratory reasoning.

\subsection{Training Configuration}
\label{app:math_config}

All math experiments are built on the VeRL framework and trained on 16 $\times$ H20 GPUs.
We train for 6 epochs with a learning rate of $1 \times 10^{-6}$ and a sampling temperature of 1.0.
The generation batch size is 256; the training batch size is 128 with gradient accumulation over 64 steps.
For each prompt, 8 candidate responses are sampled with a maximum completion length of 8{,}192 tokens.
clipping range $[\epsilon_{\mathrm{low}}, \epsilon_{\mathrm{high}}] = [0.2, 0.28]$ follows DAPO, without KL penalty.

For \methodname{}, the polarity weight bounds are set to $[\omega_{\min}, \omega_{\max}] = [0.98, 1.02]$ for the 14B backbone and $[0.98, 1.03]$ for the 7B backbone.
The entropy-slope EMA uses $\beta_{\mathrm{warm}} = 0.95$ during the warmup phase and $\beta_{\mathrm{run}} = 0.9$ thereafter.
The warmup phase lasts $N_{\mathrm{w}} = 20$ steps, during which training runs with neutral weights ($\omega_{\mathrm{pos}} = \omega_{\mathrm{neg}} = 1$) to estimate the reference collapse rate $s_{\mathrm{ref}}$.
The gate threshold is set to $\gamma_{\mathrm{gate}} = 0.3$: once the smoothed entropy falls below $0.3 \cdot h_{\mathrm{ref}}$, the adaptive reweighting deactivates permanently.
All baselines (DAPO, Ent-Reg, 80/20) share the same training budget, batch sizes, and evaluation protocol to ensure a fair comparison.

Agentic experiments are based on Tool-N1 dataset for 7 epochs with a learning rate of $1 \times 10^{-6}$ and a sampling temperature of 0.7.
The generation batch size is 256; the training batch size is 128 with gradient accumulation over 64 steps.
Group size is 5 and with a maximum completion length of 8{,}192 tokens.
For \methodname{}, the polarity weight bounds are set to $[\omega_{\min}, \omega_{\max}] = [0.99, 1.05]$.
The entropy-slope EMA uses $\beta_{\mathrm{warm}} = 0.95$ during the warmup phase and $\beta_{\mathrm{run}} = 0.9$ thereafter.
The warmup phase lasts $N_{\mathrm{w}} = 25$ steps, during which training runs with neutral weights ($\omega_{\mathrm{pos}} = \omega_{\mathrm{neg}} = 1$) to estimate the reference collapse rate $s_{\mathrm{ref}}$.

\subsection{Details of Single-Polarity Ablation}
\label{app:single_polarity}

To verify the functional meaning of effective polarity, we conduct the single-polarity ablation shown in Figure~\ref{fig:polarity_directional_control}. For each sampled token update, we compute the effective polarity $\mathcal{P}(s_t,y_t,A)=A\mathcal{T}(s_t,y_t).$
The positive-only variant keeps token-level policy-gradient terms with \(\mathcal{P}>0\) and masks out terms with \(\mathcal{P}<0\). The negative-only variant keeps terms with \(\mathcal{P}<0\) and masks out terms with \(\mathcal{P}>0\). This diagnostic ablation is implemented on vanilla GRPO rather than the DAPO-based recipe used in the main experiments. This choice isolates the effect of polarity sign and avoids confounding from DAPO-specific sampling or clipping mechanisms and from \methodname{}'s adaptive reweighting. Apart from this objective choice and the polarity mask, rewards, advantage normalization, and other hyperparameters follow the main math experimental setup.

\begin{table*}[t]
\centering
\caption{\textbf{Performance (\%) on out-of-domain tasks.}}
\label{tab:ood_detailed}
\resizebox{\textwidth}{!}{
\begin{tabular}{l|ccc|ccc|ccc}
\toprule[1.6pt]
& \multicolumn{3}{c|}{\textbf{CRUX}} & \multicolumn{3}{c|}{\textbf{IFEval}} & \multicolumn{3}{c}{\textbf{MMLU-Pro}} \\
\cmidrule{2-4} \cmidrule{5-7} \cmidrule{8-10}
\textbf{Methods} & Pass@1 & Pass@5 & Pass@10 & Pass@1 & Pass@5 & Pass@10 & Pass@1 & Pass@5 & Pass@10 \\
\midrule
\multicolumn{10}{c}{\textit{\textbf{Qwen-2.5-7B-Base}}} \\
\midrule
Base    & 16.25 & 47.00 & 60.62 & 27.54 & 39.56 & 43.81 & 36.98 & 65.33 & 75.20 \\
DAPO    & 30.63 & 61.12 & 70.38 & 28.28 & 53.79 & 61.18 & 37.81 & 70.82 & 81.83 \\
Ent-Reg & 29.75 & 54.62 & 66.50 & 27.54 & 54.16 & 62.85 & 37.02 & 70.11 & 82.33 \\
80/20   & 28.62 & 60.75 & 68.25 & 29.94 & 52.13 & 60.26 & 36.92 & 69.95 & 81.92 \\
PAPO    & 30.88 & 61.00 & 69.62 & 29.57 & 52.31 & 60.81 & 37.81 & 71.16 & 82.25 \\
\specialrule{1pt}{0.4ex}{0.4ex}
\multicolumn{10}{c}{\textit{\textbf{Qwen-2.5-14B-Base}}} \\
\midrule
Base    & 28.38 & 60.75 & 70.12 & 30.31 & 58.23 & 66.91 & 47.33 & 77.35 & 86.12 \\
DAPO    & 46.38 & 73.12 & 80.38 & 33.27 & 57.12 & 67.84 & 48.00 & 77.20 & 86.23 \\
Ent-Reg & 46.25 & 70.75 & 79.12 & 33.46 & 58.04 & 68.76 & 48.15 & 77.84 & 86.46 \\
80/20   & 47.12 & 73.25 & 78.25 & 33.64 & 58.04 & 66.54 & 48.25 & 77.60 & 86.07 \\
PAPO    & 48.25 & 74.62 & 80.12 & 33.83 & 57.12 & 66.54 & 47.89 & 77.65 & 86.02 \\
\bottomrule[1.6pt]
\end{tabular}
}
\vspace{-1em}
\end{table*}

\subsection{Details of Polarity-Aware Strength Ablation}
\label{app:strength_ablation}

This ablation (Figure~\ref{fig:ablation}) examines how the polarity weight bounds $[\omega_{\min}, \omega_{\max}]$ affect the exploration--exploitation trade-off during training.
All three \methodname{} configurations use Qwen2.5-7B-Base on the same math training setup described in Appendix~\ref{app:math_config}; the only difference is the polarity weight range:

\begin{itemize}[leftmargin=1.5em, itemsep=2pt]
  \item \textbf{Explore-biased:} $[\omega_{\min}, \omega_{\max}] = [0.99, 1.06]$. The wider upper bound amplifies the expansion branch while the near-unity lower bound leaves the contraction branch almost unmodified.
  \item \textbf{Exploit-biased:} $[\omega_{\min}, \omega_{\max}] = [0.96, 1.00]$. The sub-unity upper bound prevents any amplification of the expansion branch, while the lower bound actively suppresses it, allowing the contraction branch to dominate.
  \item \textbf{Moderate:} $[\omega_{\min}, \omega_{\max}] = [0.98, 1.03]$. This is the same setting used in the 7B main experiment.
\end{itemize}

\noindent The DAPO baseline shares the same training budget and hyperparameters but applies no polarity reweighting ($\omega_{\mathrm{pos}} = \omega_{\mathrm{neg}} = 1$ throughout).

\subsection{Prompt Templates}
\label{app:prompt_templates}

\paragraph{Math reasoning.}
All math experiments use the Qwen chat template with the following prompt format:

\begin{tcolorbox}[colback=gray!5, colframe=gray!50, boxrule=0.3pt, left=4pt, right=4pt, top=4pt, bottom=4pt]
\small\ttfamily
<|im\_start|>system\\
You are a helpful assistant.<|im\_end|>\\
<|im\_start|>user\\
\{input\}\\
Please reason step by step, and put your final answer within \textbackslash boxed\{\}.<|im\_end|>\\
<|im\_start|>assistant
\end{tcolorbox}

\noindent where \texttt{\{input\}} is replaced by the problem statement from DAPO-Math-17K.

\paragraph{Agentic Reasoning}
All agentic experiments use the Qwen chat template with the following prompt format:
\begin{tcolorbox}[colback=gray!5, colframe=gray!50, boxrule=0.3pt, left=4pt, right=4pt, top=4pt, bottom=4pt]
\small\ttfamily
$<$|im\_start|$>$system\\
You are an expert in composing functions. You are given a question and a set of possible functions. Based on the question, you will need to make one or more function/tool calls to achieve the purpose. If none of the function can be used, point it out. If the given question lacks the parameters required by the function, also point it out. You should only return the function call in tools call sections. Here is a list of functions in JSON format that you can invoke:\\
\# Tool\\
\\
You are provided with function signatures within $<$tools$>$$<$/tools$>$ XML tags:\\
\\
$<$tools$>$\\
\{tools\}\\
$<$/tools$>$\\
\\
In each action step, you MUST:\\
1. Think about the reasoning process in the mind and enclosed your reasoning within $<$think$>$$<$/think$>$ XML tags.\\
2. Then, provide a json object with function names and arguments within $<$tool\_call$>$$<$/tool\_call$>$ XML tags. i.e., $<$tool\_call$>$[\{``name'': $<$function-name$>$, ``arguments'': $<$args-json-object$>$\}, \{``name'': $<$function-name2$>$, ``arguments'': $<$args-json-object2$>$\}, ...]$<$/tool\_call$>$\\
3. Make sure both the reasoning and the tool call steps are included together in one single reply.\\
A complete reply example is: $<$think$>$ To address the query, I need to send the email to Bob and then buy the banana through walmart. $<$/think$>$ $<$tool\_call$>$ [\{``name'': ``email'', ``arguments'': \{``receiver'': ``Bob'', ``content'': ``I will bug banana through walmart''\}\}, \{``name'': ``walmart'', ``arguments'': \{``input'': ``banana''\}\}]$<$/tool\_call$>$. Please make sure the type of the arguments is correct.$<$|im\_end|$>$\\
$<$|im\_start|$>$user\\
\{input\}\\$<$|im\_end|$>$\\
$<$|im\_start|$>$assistant

\end{tcolorbox}
\noindent where \texttt{\{tools\}} and \texttt{\{input\}} are replaced by the tool list and the problem statement from the Tool-N1 dataset, respectively.

\section{Additional Results}
\label{sec:additional_results}

\subsection{Positional Distribution of Entropy Polarity}
\label{app:polarity_position}
We explore how $\mathcal{P}$ varies with token position within each rollout in Figure~\ref{fig:t12_position}.

\paragraph{Magnitude distribution.}
The mean $|\mathcal{P}|$ remains nearly constant across all positions ($\approx 0.04$), while the median sits approximately three orders of magnitude lower (${\sim}10^{-5}$), confirming the heavy-tail sparsity reported in Section~\ref{sec:theory_verification}.
This gap implies that more than half of all tokens at every position contribute negligibly to entropy dynamics; the entropy-shaping effect is concentrated in a sparse tail of high-$|\mathcal{P}|$ tokens.
The median exhibits a mild boundary effect---elevated at the start and end of the rollout (${\sim}10^{-2}$ to $10^{-4}$) compared to ${\sim}10^{-5}$ in the interior---suggesting that tokens near sequence boundaries carry moderately higher polarity on average.
Crucially, however, the flat mean demonstrates that the dominant entropy-shaping tokens are distributed uniformly across all positions.
This justifies a position-agnostic reweighting scheme: methods that filter tokens by position would discard informative gradient signal at every location, whereas polarity-based selection captures the sparse, high-impact tokens regardless of where they appear.

\paragraph{Signed polarity reveals a positional trend.}
The signed mean transitions from net negative (contraction-dominated) in the first ${\sim}15\%$ of the sequence to net positive (expansion-dominated) thereafter, while the median remains at zero throughout.
This pattern admits a natural interpretation: early tokens in reasoning chains typically establish the solution direction (e.g., selecting a proof strategy or computation framework) and tend to be high-confidence, so reinforcement concentrates probability further and produces a contracting effect.
Later tokens handle more uncertain, problem-specific reasoning steps where the model distributes probability more broadly, yielding a net expansion tendency.
The near-zero median confirms that expansion and contraction tokens are balanced in count at every position; the directional bias arises solely from the magnitude-weighted tail, consistent with the heavy-tail structure observed in panel~(a).

\begin{figure}[t]
\centering
\includegraphics[width=1.0\textwidth]{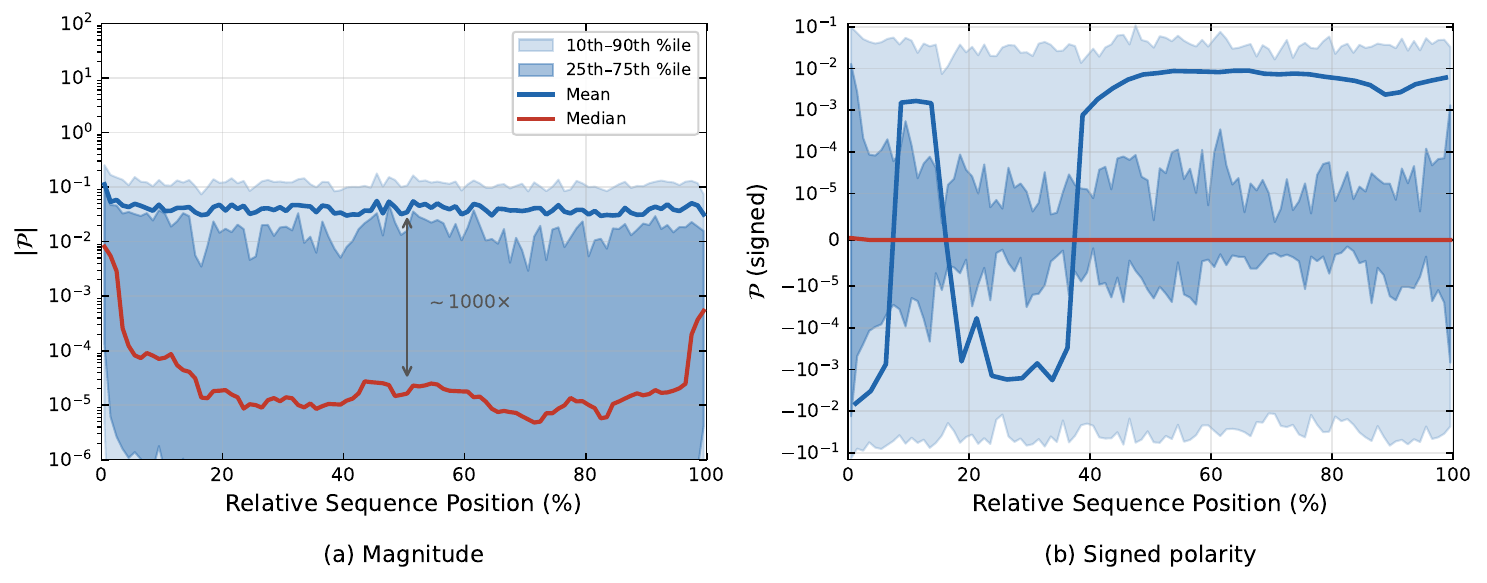}
\caption{Positional distribution of entropy polarity $\mathcal{P}$. Shaded bands show the 10th--90th and 25th--75th percentile ranges; solid lines show mean (blue) and median (red). (a)~Magnitude $|\mathcal{P}|$: the mean remains flat while the median sits ${\sim}1000\times$ lower, reflecting heavy-tail sparsity at every position. (b)~Signed $\mathcal{P}$: the mean transitions from net contraction in early positions to net expansion in later positions, while the median remains at zero.}
\label{fig:t12_position}
\end{figure}

\subsection{Complete Training Dynamics}
\label{sec:appendix_training_dynamics}

\begin{figure*}[h]
    \centering
    \includegraphics[width=\linewidth]{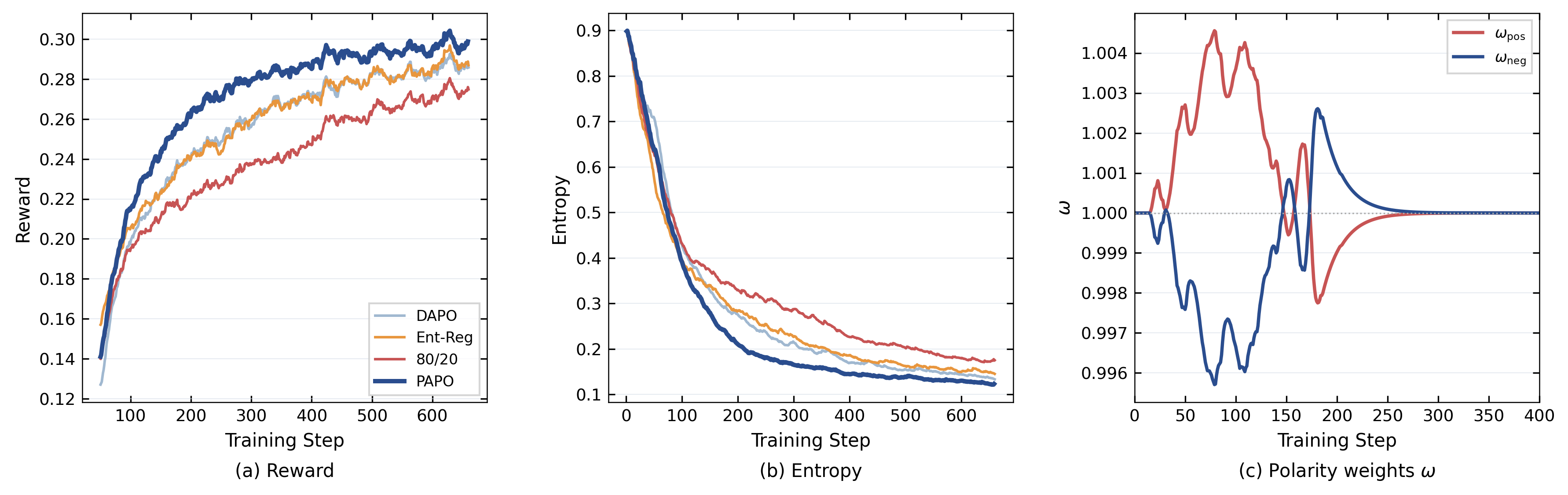}
    \caption{\textbf{Training dynamics on Qwen2.5-7B-Base (Math).}
    (a)~Reward, (b)~entropy, and (c)~polarity weights $\omega$ for all methods trained on DAPO-Math-17k.
    Trends are consistent with the 14B results in Figure~\ref{fig:math_reward_entropy}.}
    \label{fig:7b_curves}
\end{figure*}

\paragraph{Qwen2.5-7B-Base (Math).}
Figure~\ref{fig:7b_curves} shows the full training curves for all four methods on the 7B backbone.
The overall trends mirror the 14B results in Figure~\ref{fig:math_reward_entropy}: \methodname{} achieves the highest final reward with a stable entropy trajectory.
Unlike at 14B, \methodname{}'s entropy settles slightly below DAPO, suggesting that the smaller model naturally favors a more exploitative regime; crucially, entropy remains stable without collapse, indicating sufficient exploratory capacity for continued learning.
The adaptive weights $\omega$ exhibit the same transient rebalancing pattern observed at 14B scale.

\begin{figure*}[h]
    \centering
    \includegraphics[width=\linewidth]{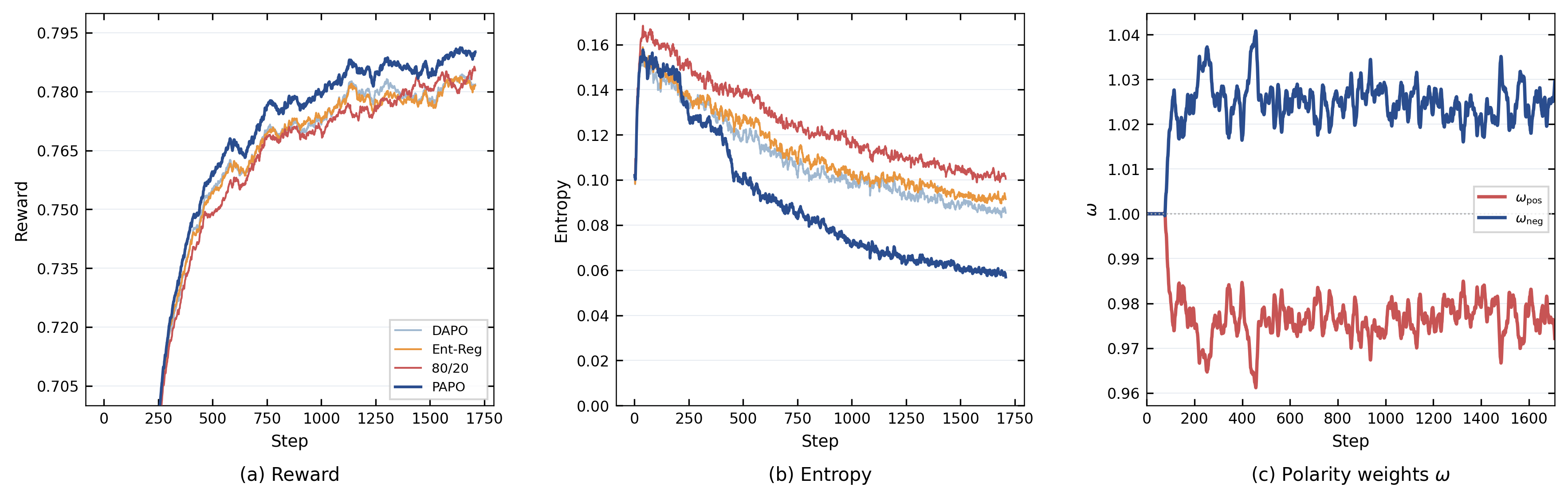}
    \caption{\textbf{Training dynamics on tool-call reasoning (Qwen2.5-14B-Instruct).}
    (a)~Reward, (b)~entropy, and (c)~polarity weights $\omega$.
    The adaptive weights consistently favor the contracting branch, reflecting the more constrained action space of structured tool-call generation.}
    \label{fig:toolcall_curves}
\end{figure*}

\paragraph{Qwen2.5-14B-Instruct (Agentic).}
Figure~\ref{fig:toolcall_curves} presents the training dynamics for the agentic setting.
Unlike math reasoning, the adaptive weights consistently favor the contracting branch ($\omega_{\mathrm{neg}} > \omega_{\mathrm{pos}}$), reflecting the more constrained action space of structured tool-call generation where exploitation is naturally preferred.

\subsection{Detailed Out-of-Domain Results}
\label{sec:appendix_ood}

Table~\ref{tab:ood_detailed} reports per-benchmark OOD results for all methods at both scales.
All models are trained exclusively on DAPO-Math-17k and evaluated zero-shot on code reasoning, instruction following, and general knowledge benchmarks.

\end{document}